\documentclass[letterpaper, 10 pt, conference]{ieeeconf}  

\IEEEoverridecommandlockouts                              
\let\labelindent\relax
\usepackage{enumitem}
\usepackage{graphics} 
\usepackage{epsfig} 
\usepackage{amsmath} 
\usepackage{amssymb}  

\usepackage{amsthm}
\usepackage{bm}
\usepackage{algorithm}
\usepackage{algorithmic}
\usepackage{cite}
\usepackage{dsfont}
\usepackage[citecolor=blue, colorlinks=true, linkcolor=blue, urlcolor=black]{hyperref} 
\newcommand\numberthis{\addtocounter{equation}{1}\tag{\theequation}}

\makeatletter
\newcommand\fs@betterruled{%
  \def\@fs@cfont{\bfseries}\let\@fs@capt\floatc@ruled
  \def\@fs@pre{\vspace*{5pt}\hrule height.8pt depth0pt \kern2pt}%
  \def\@fs@post{\kern2pt\hrule\relax}%
  \def\@fs@mid{\kern2pt\hrule\kern2pt}%
  \let\@fs@iftopcapt\iftrue}
\floatstyle{betterruled}
\restylefloat{algorithm}
\makeatother

\theoremstyle{plain}
\newtheorem{theorem}{Theorem}
\newtheorem{lemma}[theorem]{Lemma}  
\newtheorem{corollary}[theorem]{Corollary}

\theoremstyle{definition}

\newtheorem{assumption}[theorem]{Assumption}

\newtheorem{remark}[theorem]{Remark}

\usepackage[svgnames]{xcolor}

\title{\LARGE \bf
Large Stepsizes Federated Learning on Logistic Regression with Linearly Separable Data: The Case of Heterogeneous Devices
}

\author{Hok Fong Wong, Hoi-To Wai, and Chung-Yiu Yau
\thanks{Hok Fong Wong is with Department of CSE, The Chinese University of Hong Kong, Hong Kong SAR
      {\tt\small hfwong2@cse.cuhk.edu.hk}.}%
\thanks{Hoi-To Wai is with Department of SEEM, The Chinese University of Hong Kong, Hong Kong SAR
      {\tt\small htwai@se.cuhk.edu.hk}.}%
\thanks{Chung-Yiu Yau is with Department of ECE, University of Minnesota, Minnesota, USA
      {\tt\small cyau@umn.edu}.}%
}

\begin{document}

\maketitle
\thispagestyle{empty}
\pagestyle{empty}

\begin{abstract}
This paper revisits the distributed learning problem for training a multinomial logistic regression model with the Federated Averaging ({\tt FedAvg}) algorithm. We concentrate on a scenario with arbitrarily large stepsizes and heterogeneous update rules where the devices may perform a different number of local updates in each round. We show that, with linearly separable data, {\tt FedAvg} is stable with any stepsizes and the objective values converge to zero at the rate of ${\cal O}(1/R)$, where $R$ is the number of communication rounds. Our result also demonstrates that the effects of device heterogeneity vanish asymptotically. For sufficiently large $R$, the objective values decrease monotonically and is bounded by ${\cal O}( 1 / (R T_{\rm avg}))$, where $T_{\rm avg}$ is the average number of local update steps per communication round across devices. Numerical experiments support our findings. 
\end{abstract}

\section{Introduction} 
Distributed optimization has been a major workhorse in multi-agent control and learning systems \cite{tsitsiklis2003distributed, nedic2018distributed, yang2019survey}, where devices (a.k.a.~agents) collaboratively optimize a global objective function by interleaving between local computation and communication. In recent years, federated learning (FL) has gained popularity as a mainstream distributed optimization paradigm as it offers a disciplined way to account for deployment reality, where heterogeneous devices differing in hardware, battery, connectivity operates in an asynchronous manner \cite{mcmahan2017communication, kairouz2021advances}. 
Such aspect has become increasingly relevant to multi-agent systems  \cite{weber2024combining} which are intrinsically heterogeneous due to hardware differences. FL has been successfully applied in problems such as system identification \cite{wang2023fedsysid, chen2023multi}, reinforcement learning and control \cite{liu2023distributed}.

The challenges of heterogeneous devices have set the study of FL apart from traditional distributed optimization by workers performing local updates before global aggregation of their local solutions, introducing \emph{local drift} error that results from data and computational heterogeneity. 
The issue of local drift in FL has inspired the studies of algorithms such as {\tt FedAvg} \cite{li2019convergence, li2020federated,wang2021novel}, {\tt SCAFFOLD} \cite{karimireddy2020scaffold}, {\tt FedPD} \cite{zhang2021fedpd}, and {\tt FedProx} \cite{yuan2022convergence}. For instance, {\tt FedAvg} relies on a simple averaging of local updates, and {\tt SCAFFOLD}, {\tt FedPD}, {\tt FedProx} utilize control variates or proximal updates to reduce local drift errors. Convergence is guaranteed together with sufficiently small stepsizes with respect to (w.r.t.) problem-specific parameters such as smoothness. 

On the other hand, several recent researches found that gradient algorithms may still converge with large stepsizes \cite{wu2024large, wu2025large, tyurin2025logistic, meng2024gradient}. These pioneering works showed that for logistic regression with linearly separable data, gradient descent with arbitrarily large stepsizes stops oscillating after a transient phase, and even exhibits an accelerated convergence rate. This marks a significant departure from the classical analysis that demands small stepsizes to ensure monotonic descent.

In light of the above developments, this paper investigates the impact of large stepsizes on FL. We revisit the convergence of {\tt FedAvg} with large stepsizes for logistic regression with linearly separable data. We are particularly interested in the effects of device computational heterogeneity, i.e., when the number of local update steps varies across devices, on the robustness of {\tt FedAvg} with large stepsizes. Notice that a closely related work is the recent paper \cite{crawshaw2025constant} which studied a special case of {\tt FedAvg} with only homogeneous devices taking the same number of local updates. We inquire:
\begin{center}
    \textit{Does the convergence of {\tt FedAvg} with large stepsizes still hold under heterogeneous local updates, and how does device heterogeneity affect the convergence bound?}
\end{center}

The practical relevance of our study extends further to federated control with reinforcement learning (FCRL) where a central server coordinates distributed agents to learn policies without sharing sensitive experience data \cite{kumar2017federated, jin2022federated}. Our setting captures the core update structure of federated softmax policy gradient methods under computational heterogeneity. From a networked control perspective, our setup models a high-gain feedback loop with heterogeneous actuators with different duty cycles. It is thus natural to consider whether the closed loop remains stable, and how the resulting dynamics depend on the spread of the numbers of local update steps.

To address the above question, we notice that the introduction of \emph{heterogeneous devices} in local updates in {\tt FedAvg} has led to a nontrivial complex dynamics with large stepsizes. Our main contributions are summarized below as: 
\begin{itemize}[leftmargin=*]
      \item Through utilizing a new proof technique, we show that {\tt FedAvg} with large stepsizes and heterogeneous devices is stable and the objective values converge to zero at ${\cal O}(1/R)$, where $R$ is the number of communication rounds.
      \item We demonstrate that the effects of device heterogeneity on the convergence rate vanish asymptotically. For $R \gg 1$, the convergence of optimality gap in objective values behaves as ${\cal O}( 1 / (R T_{\rm avg}))$, where $T_{\rm avg}$ is the average number of local update steps per communication round.
      \item Our numerical experiments demonstrate that the asymptotic convergence of {\tt FedAvg} with large stepsizes is not sensitive to the degree of device heterogeneity, corroborating with our theoretical findings. 
\end{itemize} 
We also developed our analysis for the multinomial logistic regression setting, extending from prior works such as \cite{wu2024large}.\vspace{.1cm}

\noindent {\bf Notations.} For a vector $\mathbf{v}$, $\lVert \mathbf{v} \rVert$ denotes its $\ell_2$-norm. For a matrix $\mathbf{A}$, $\lVert \mathbf{A} \rVert_F$ and $\lVert \mathbf{A} \rVert_{\text{op}}$ denote its Frobenius norm and operator norm, respectively. We use $[n]$ to denote the set $\{1, 2, \ldots, n\}$ for any positive integer $n$. We use ${\cal O}(\cdot)$ to hide absolute constants independent of problem parameters.

\section{Problem Setup and FedAvg Algorithm}
Consider a distributed optimization setup with a central server and $m$ devices. Each device $i$ has a local dataset $\mathcal{D}_i = \{ ( {\bm x}_{i,j}, y_{i,j} ) \}_{j=1}^{n_i}$, where ${\bm x}_{i,j} \in \mathbb{R}^d$ is the feature vector and $y_{i,j} \in [K]$ is the corresponding label. Parameterizing the multinomial linear model by $\mathbf{W}=[\mathbf{w}_1^\top;\mathbf{w}_2^\top; \ldots; \mathbf{w}_K^\top]\in\mathbb{R}^{K\times d}$, our goal is to tackle the following problem:
\begin{equation}
      \label{eq:opt}
      \inf_{ {\bf W} \in \mathbb{R}^{K \times d} } {\cal L}({\bf W}) := \frac{1}{n} \sum_{i=1}^m n_i {\cal L}_i( {\bf W} ),
\end{equation}
where $n := \sum_{i=1}^m n_i$, ${\cal L}_i( {\bf W} ) := \frac{1}{n_i} \sum_{j=1}^{n_i} \ell_{i,j}( {\bf W} )$ with
\[
\ell_{i,j}( {\bf W} ) := \log \left( \frac{1}{ \widehat{y}_{i,j | {\bf W}}} \right),~\widehat{y}_{i,j | {\bf W}} := \frac{ \exp( {\bf w}_{y_{i,j}}^\top {\bm x}_{i,j} ) }{ \sum_{k=1}^K \exp( {\bf w}_k^\top {\bm x}_{i,j} ) }.
\]
Notice that $\ell_{i,j}( {\bf W} ) \geq 0$ is the cross-entropy loss for the $j$-th data point on device $i$ and tackling \eqref{eq:opt} trains a $K$-class linear classifier ${\bf W}$. 

In this work, we concentrate on the federated averaging (\texttt{FedAvg}) algorithm \cite{mcmahan2017communication} with heterogeneous local updates. As summarized in Algorithm \ref{alg:fedavg}, the algorithm mimicks the behavior of centralized gradient descent while allowing the devices to perform local update using only their own data. The difference matrix $\Delta {\bf D}_i^{(r)}$ is controlled by the local step size $\eta_l$ and accumulates the local updates from device $i$ in each round. This difference matrix will be sent to the server at round $r$, where the latter aggregates the updates with a global stepsize $\eta_g$. Depending on the discrepencies in hardware and network conditions, the heterogeneous devices may perform different number of local updates $T_i$ in each communication round. We assume that $T_i \geq 1$ for all $i$. Notice that if $T_i = 1$ for all $i$, then the evolution of $\{ {\bf W}^{(r)} \}_{r \geq 0}$ in {\tt FedAvg} reduces to that of the (centralized) gradient descent algorithm on minimizing ${\cal L}( {\bf W})$.

\begin{algorithm}[t]
\caption{{\tt FedAvg} Algorithm}
\label{alg:fedavg}
\begin{algorithmic}[1]
\STATE {\bfseries Input:} Initial model ${\bf W}^{(0)}$, stepsizes $\eta_{l} > 0, \eta_g > 0$, maximum number of communication rounds $R$.
\FOR{$r=1$ to $R$}
\STATE Server broadcasts ${\bf W}^{(r-1)}$ to all devices.
\FOR{each device $i$ in parallel}
\STATE Set ${\bf D}_{i,0}^{(r)} = {\bf W}^{(r-1)}$.
\FOR{$t=0$ to $T_i-1$} 
\STATE Update ${\bf D}_{i,t+1}^{(r)} = {\bf D}_{i,t}^{(r)} - \eta_{l} \nabla {\cal L}_{i}( {\bf D}_{i,t}^{(r)} )$.
\ENDFOR
\STATE {Send} $\Delta {\bf D}_i^{(r)} = {\bf D}_{i,T_i}^{(r)} - {\bf W}^{(r-1)}$ to server.
\ENDFOR
\STATE Server updates ${\bf W}^{(r)} = {\bf W}^{(r-1)} +  {\eta_g} \sum_{i=1}^m \frac{n_i}{n} \Delta {\bf D}_i^{(r)}$.
\ENDFOR
\end{algorithmic}
\end{algorithm}

The convergence of {\tt FedAvg} has been widely studied, e.g., in \cite{li2019convergence, karimireddy2020scaffold, li2020federated, wang2021novel}. A main technical challenge is to account for the device drift behavior when $T_i \geq 2$ and the resultant device heterogeneity that differentiate the algorithm from that of centralized gradient descent. As such, the majority of existing analysis focused on the small stepsizes regime with stochastic gradients and identical number of local steps $T_i = T$ for all $i$. For instance, \cite[Theorem I]{karimireddy2020scaffold} requires that $\eta_g \eta_{l} = {\cal O}( 1/ (L \max_i T_i) )$, where $L$ is the smoothness constant of ${\cal L}_i(\cdot)$, and \cite{wang2021novel} demonstrated that when $T_i \neq T_j$ for some $i \neq j$, {\tt FedAvg} converges to a biased optimal solution depending on a reweighted objective function. On the other hand, the recent work \cite{crawshaw2025constant} showed a new convergence result of {\tt FedAvg} for arbitrarily large stepsizes applied to logistic regression with linearly separable data. However, a limitation of their analysis is that it only applies to homogeneous devices where $T_i = T$ for all $i$, as well as the binary classification setting with $K=2$.

The above motivated us to extend the analysis with large stepsize {\tt FedAvg} to heterogeneous devices. To describe the setup, our analysis relies on the following standing assumptions inspired by \cite{wu2024large,crawshaw2025constant}. We require:
\begin{assumption}[Bounded Feature Norms] \label{ass:bddf_dist}
    All data samples $\{\mathbf{x}_{i,j}, y_{i,j}\}$, where $i \in [n_j]$, satisfy $\lVert \mathbf{x}_{i,j} \rVert \leq 1$.
\end{assumption}
\begin{assumption}[Global Linear Separability] \label{ass:linsep_dist}
    There exists $\gamma > 0$, $\mathbf{W}_\star=[\mathbf{w}_{\star,1}^\top; \mathbf{w}_{\star,2}^\top; \ldots; \mathbf{w}_{\star,K}^\top]$ with $\lVert \mathbf{W}_{\star}\rVert_F \leq 1$ such that for all data samples $(\mathbf{x}_{i,j}, y_{i,j})$ and $k \neq y_{i,j}$, 
    \[
    (\mathbf{w}_{\star,y_{i,j}}-\mathbf{w}_{\star,k})^\top\mathbf{x}_{i,j}\geq \gamma,~\forall~i,j. \vspace{.2cm}
    \] 
\end{assumption}
Assumption \ref{ass:bddf_dist} is a standard regularity condition to avoid potentially unbounded feature vectors. On the other hand, Assumption \ref{ass:linsep_dist} requires that the training data samples are linearly separable by a global solution matrix ${\bf W}_\star$. Note that this implies $\inf_{ {\bf W} } {\cal L}_i({\bf W}) = 0$ for all $i$ and $\inf_{ {\bf W} } {\cal L}({\bf W}) = 0$, though the latter may only be approximated by a scaled version of ${\bf W}_\star$. We remark that the linearly separability condition is essential for large-stepsize convergence to the minimizers\cite{wu2024large,crawshaw2025constant, tyurin2025logistic,wu2025large}, as the non-separable case is known to admit stable cycles, see \cite{meng2024gradient, meng2025gradient} for recent developments on centralized gradient descent.

\section{Large Stepsize Convergence Analysis}
This section presents the main convergence results for large stepsize {\tt FedAvg} and provide details for their proofs.

Define the following constants on the number of local update steps taken by the devices
    \[
    T_{\rm avg} :=\sum_{i=1}^m \frac{n_i T_i}{n},~~ T_{\max}:=\max_{i\in[m]} T_i,~~ T_{\min}:=\min_{i\in[m]} T_i.
    \]
We also set $\bar{\eta} := \eta_g \eta_l$ as the product of local and global stepsizes.
Note that $1 \leq T_{\min} \leq T_{\rm avg} \leq T_{\max}$.
Our first result is the following bounds on the average risk and magnitude of the decision variable, i.e., 
\begin{theorem}\label{thm:eos_average_loss_bound}
    Under Assumptions \ref{ass:bddf_dist}, \ref{ass:linsep_dist}. For
    every $\eta_g, \eta_l>0$ and $R \geq 1$, we have
    \begin{align}
        &\frac{1}{R}\sum\limits_{r=0}^{R-1}\mathcal{L}(\mathbf{W}^{(r)})\leq \frac{\lVert\mathbf{W}^{(0)}\rVert_F^2}{\bar{\eta}R T_{\min}} \label{eq:risk_bd}\\
        & + \frac{2(1 + \log^2(\gamma^2\bar{\eta} R(K-1)T_{\rm avg})+\bar{\eta}^2T_{\rm avg}^2+{\eta}_l^4 T_{\max}^4)}{\bar\eta \, \gamma^2 R T_{\min}}. \nonumber
    \end{align}
    Moreover,
    \begin{align}
    \lVert \mathbf{W}^{(R)}\rVert_F &\leq \lVert\mathbf{W}^{(0)}\rVert_F +\frac{\sqrt{2}+\log(\gamma^2\bar{\eta} R(K-1)T_{\rm avg})}{\gamma} \nonumber \\
    &\quad +\frac{\bar{\eta}T_{\rm avg}+{\eta}_l^2 T_{\max}^2}{\gamma}. \label{eq:param_bound}
    \end{align}
\end{theorem}
Unlike the classical analysis of {\tt FedAvg} \cite{karimireddy2020scaffold,li2019convergence, wang2021novel}, the bounds \eqref{eq:risk_bd}, \eqref{eq:param_bound} hold \emph{for any} $\eta_g , \eta_l > 0$ and any distribution of $\{T_i\}_{i=1}^m$. It demonstrates that the average risk converges to zero at ${\cal O}(1/R)$ with heterogeneous devices. From \eqref{eq:risk_bd}, \eqref{eq:param_bound}, we notice that the effect of device heterogeneity is captured in $\eta_l^4 \frac{ T_{\max}^4 }{T_{\min}}$, $ \bar{\eta}^2 \frac{ T_{\rm avg}^2 }{T_{\min}}$. Here, \eqref{eq:risk_bd} shows that larger stepsizes $\eta_l, \bar{\eta}$ can amplify the effect of device heterogeneity. Interestingly, our result also indicates that choosing a smaller local stepsize $\eta_l$ may improve the performance. 

In comparison, the recent work \cite[Theorem 4.1]{crawshaw2025constant} showed a similar bound for {\tt FedAvg} as \eqref{eq:risk_bd} under the case of homogeneous devices with $T_i = T$ and binary classification. Similar to \cite{wu2024large}, they showed that
\[
\frac{1}{R}\sum_{r=0}^{R-1}\mathcal{L}(\mathbf{W}^{(r)}) = {\cal O}\left( \frac{ \log^2(T (1+ \bar{\eta} \gamma^2 R) ) + \bar{\eta}^2 T^2 }{ \bar{\eta} \, \gamma^4 R } \right).
\]
Considering the special case of \eqref{eq:risk_bd} with $T_i = T$ and $K=2$, we note that when $R \gg 1$, the above bound is worse than \eqref{eq:risk_bd} by a factor of $1/\gamma^2$. As $\gamma < 1$, this suggests that Theorem \ref{thm:eos_average_loss_bound} offers a tighter bound than \cite[Theorem 4.1]{crawshaw2025constant} even for the homogeneous devices case.
We remark that the improvement of \eqref{eq:risk_bd} is potentially due to the new proof structure which we shall develop in Section~\ref{sec:proof_idea}.

We next show that under an additional condition on the degree of device heterogeneity, {\tt FedAvg} exhibits a \emph{stable phase} convergence behavior:
\begin{theorem}
    \label{thm:stablep}
    Under Assumptions \ref{ass:bddf_dist} and \ref{ass:linsep_dist}, for every $\eta_g, \eta_l>0$, suppose that there exists $s \geq 0$ such that
    $\mathcal{L}(\mathbf{W}^{(s)})\leq \min\{ \frac{1}{2n}, \frac{1}{4\bar{\eta}T_\text{avg}}, \frac{\gamma T_{\rm avg}}{40n\eta_l  T_{\rm max}^2 \sqrt{K}} \}$,
    and the heterogeneity among clients satisfies $\max_i |T_i/T_{\rm avg}-1| \leq \gamma/(20\sqrt{K})$,
    then 
    \[
    \mathcal{L}(\mathbf{W}^{(r+1)}) \leq \mathcal{L}(\mathbf{W}^{(r)}),~\forall~r \geq s. \numberthis \label{eq:monotone_result}
    \] 
\end{theorem}

\begin{remark}
    The condition $\max_i|T_i/T_\mathrm{avg}-1|\leq \gamma/(20K)$ is sufficient but conservative. It bounds the heterogeneity error $\widetilde{\mathbf{E}}^{(r)}$, a zero-sum combination of local gradients depending only on their deviations from the average. The proof bounds these deviations by $\|\nabla \mathcal{L}_i \|+\|\nabla \mathcal{L}\|$, which is tight only when the local gradients are maximally misaligned.
\end{remark}

The above observation shows that once ${\cal L}( {\bf W}^{(r)})$ is smaller than a threshold, the {\tt FedAvg} algorithm enters a stable phase where the objective value decreases monotonically. Notice that \eqref{eq:risk_bd} guarantees that as $R \gg 1$, there exists $s \in [R]$ such that ${\cal L}( {\bf W}^{(s)} ) = {\cal O}(1/R)$ and thus the existence of such iteration index $s$.
Theorem \ref{thm:stablep} further leads to an asymptotic convergence rate for {\tt FedAvg}: 
\begin{corollary} \label{thm:stabled}
    Under the same assumptions as Theorem \ref{thm:stablep}, 
    \[
    \mathcal{L}(\mathbf{W}^{(r)})\leq \frac{16}{\bar{\eta} T_{\rm avg} \gamma^2(r-s)},~\forall~r\geq s. \numberthis \label{eq:last_iter_risk}
    \]
\end{corollary}
The above result strengthens the average risk bound in \eqref{eq:risk_bd} and shows that the last iterate ${\cal L}( {\bf W}^{(r)} )$ converges as ${\cal O}(1/ (r T_{\rm avg}))$ when $r \gg 1$. Importantly, the effects of device heterogeneity vanish asymptotically and the convergence rate is dominated by the average local update steps $T_{\rm avg}$. Such estimation on the convergence rate is consistent with \cite{wang2021novel} which showed that {\tt FedAvg} with small stepsizes converges at a similar rate dependent on $T_{\rm avg}$. Unlike \cite{wang2021novel}, {\tt FedAvg} in our setting converges to an unbiased solution as $\{ {\cal L}_i(\cdot) \}_{i=1}^m$ admit a common optimal solution under Assumption \ref{ass:linsep_dist}.
We remark that the above results assume the distribution of $\{T_i\}_{i=1}^m$ to satisfy $\max_i |T_i/T_{\rm avg}-1| \leq \gamma/(20\sqrt{K})$, e.g., when $T_{\rm avg}$ is large or the variation of $\{T_i\}_{i=1}^m$ is small. As shown in Section~\ref{sec:exp}, such stable phase convergence behavior is observed in practice even without the above condition. 

We present proof outlines for the theorems above. The missing proofs are in \url{https://www1.se.cuhk.edu.hk/~htwai/pdf/cdc26-largestep.pdf}.

\subsection{Proof of Theorem \ref{thm:eos_average_loss_bound}} \label{sec:proof_idea}
Let ${\bf W}_\star$ be a matrix satisfying Assumption \ref{ass:linsep_dist}. To facilitate our proof, for any $q \in [K]$, let ${\bf W} := [ {\bf w}_1, \ldots, {\bf w}_K ]$,
\[ \textstyle
\widehat{y}_{i,j | {\bf W}}^{(q)} := \frac{ \exp( {\bf w}_{q}^\top {\bm x}_{i,j} ) }{ \sum_{k=1}^K \exp( {\bf w}_k^\top {\bm x}_{i,j} ) } .
\]
Notice that $\sum_{k=1}^K \widehat{y}_{i,j | {\bf W}}^{(k)} = 1$.
We first observe the following consequence of Assumption \ref{ass:linsep_dist}, summarized as the lemma:
\begin{lemma} \label{lem:linsep} Under Assumption \ref{ass:linsep_dist}.
We observe the following bounds for any $i,j$,
\begin{enumerate}[label=(\roman*)]
    \item with any $\alpha > 0$, 
    \begin{equation}
        \ell_{i,j} ( \alpha {\bf W}_\star ) \leq (K-1)\exp\left( -\alpha \gamma \right). \label{eq:lij_bound}
    \end{equation}
    \item with any ${\bf W}$,
    \begin{equation}
        \langle \nabla \ell_{i,j} ( {\bf W} ), {\bf W}_\star \rangle \leq - \gamma ( 1-\widehat{y}_{i,j | {\bf W}} ). \label{eq:lij_gradient_match}
    \end{equation}
    \item with any ${\bf W}$, 
    \begin{equation}
        \| \nabla \ell_{i,j} ( {\bf W} ) \|_F^2 \leq 2 ( 1- \widehat{y}_{i,j | {\bf W}} ). \label{eq:lij_gradient_bound}
    \end{equation}
\end{enumerate}
\end{lemma}
\noindent The above results are consequences of Assumption \ref{ass:linsep_dist} and properties of the cross-entropy loss.

Our key step is to control the distance between ${\bf W}^{(r)}$ and ${\bf U} := {\bf U}_1 + {\bf U}_2 + {\bf U}_3$, defined as 
\begin{align*}
    \mathbf{U}_1=\frac{\log((K-1)\gamma^2\bar{\eta} R T_{\rm avg} )}{\gamma}\mathbf{W}_\star, \\\mathbf{U}_2=\frac{\bar{\eta}T_{\rm avg}}{\gamma}\mathbf{W}_\star, \mathbf{U}_3=\frac{\eta_l^2 T_{\max}^2}{\gamma}\mathbf{W}_\star.\numberthis
\end{align*}
In particular, we notice that the server update in round $r$, $\Delta {\bf W}^{(r)} = {\bf W}^{(r+1)} - {\bf W}^{(r)}$, can be written as
\begin{equation}
\bm{\Delta} \mathbf{W}^{(r)} = 
-\frac{\bar{\eta}}{n}\sum_{i=1}^m\sum_{\tau=1}^{T_i} n_i\cdot\nabla\mathcal{L}_i(\mathbf{D}_{i,\tau-1}^{(r)}).
\end{equation} 
As such, 
\begin{align*}
& \| {\bf W}^{(r+1)} - {\bf U} \|_F^2 - \| {\bf W}^{(r)} - {\bf U} \|_F^2 = \| \bm{\Delta} \mathbf{W}^{(r)} \|_F^2 \numberthis \\
& + \underbrace{ 2\langle \bm{\Delta} \mathbf{W}^{(r)} , {\bf W}^{(r)} - {\bf U}_1 \rangle}_{=: {\cal A}_1} + \underbrace{2\langle -\bm{\Delta} \mathbf{W}^{(r)} , {\bf U}_2 + {\bf U}_3 \rangle}_{=: {\cal A}_2} .
\end{align*}
To bound ${\cal A}_1$, we have that as ${\cal L}_i(\cdot)$ is $1/2$-smooth (cf.~Lemma \ref{lmm:smoothloc}) and convex,
\begin{align*}
& \langle \nabla \mathcal{L}_i(\mathbf{D}^{(r)}_{i,\tau-1}), \mathbf{W}^{(r)} - \mathbf{U}_1 \rangle \\
& \geq {\cal L}_i( {\bf W}^{(r)} ) - {\cal L}_i( {\bf U}_1 ) - \frac{1}{4} \| {\bf W}^{(r)} - {\bf D}_{i,\tau-1}^{(r)} \|_F^2.\numberthis
\end{align*}
Therefore,
\begin{align*}
{\cal A}_1 & = -\frac{2\bar{\eta}}{n}\sum_{i=1}^m\sum_{\tau=1}^{T_i} n_i  \langle \nabla \mathcal{L}_i(\mathbf{D}^{(r)}_{i,\tau-1}),\mathbf{W}^{(r)}-\mathbf{U}_1 \rangle \\
& \leq \frac{2\bar{\eta}}{n}\sum_{i=1}^m n_i T_i \left( \mathcal{L}_i(\mathbf{U}_1) - \mathcal{L}_i(\mathbf{W}^{(r)})  \right) + {\cal E}^{(r)} 
\end{align*} 
where 
\begin{align}
    {\cal E}^{(r)} := \frac{\bar{\eta}}{2n}\sum_{i=1}^m \sum_{\tau=1}^{T_i} n_i \|\mathbf{W}^{(r)}-\mathbf{D}^{(r)}_{i,\tau-1}\|^2_F. \numberthis \label{eq:cal_e_r}
\end{align}
Applying Lemma \ref{lem:linsep}-(i) with $\alpha = \log((K-1)\gamma^2\bar{\eta} R T_{\rm avg} )/\gamma$, it holds that
${\cal L}_i( {\bf U}_1 ) \leq 1 / ( \gamma^2 \bar{\eta} RT_{\rm avg})$ for all $i$. Therefore,
\begin{align*}
\frac{2\bar{\eta}}{n}\sum_{i=1}^m n_i T_i {\cal L}_i( {\bf U}_1 ) \leq \frac{2}{ \gamma^2 R }.\numberthis
\end{align*}  
On the other hand, with ${\cal L}_i( {\bf W}^{(r)} ) \geq 0$, we have
\begin{align*}
    \frac{2\bar{\eta}}{n}\sum_{i=1}^m n_i T_i \mathcal{L}_i(\mathbf{W}^{(r)}) \geq 2 \bar{\eta} T_{\min} {\cal L}( {\bf W}^{(r)} ). \numberthis
\end{align*}
This yields ${\cal A}_1 \leq \frac{2}{\gamma^2 R} - 2 \bar{\eta} T_{\min} {\cal L}( {\bf W}^{(r)} ) + {\cal E}^{(r)}$.

To bound ${\cal A}_2$, we can apply Lemma \ref{lem:linsep}-(ii) and observe that ${\cal A}_2$ can be evaluated and bounded as
\begin{align*}
& \frac{2\bar{\eta}}{n}\left(\frac{\bar{\eta}T_{\rm avg}}{\gamma}+\frac{\eta_l^2T_{\max}^2}{\gamma}\right)\sum_{i=1}^m\sum_{\tau=1}^{T_i}n_i \langle \nabla \mathcal{L}_i(\mathbf{D}^{(r)}_{i,\tau-1}),\mathbf{W}_\star\rangle\\
&\leq \frac{2\bar{\eta}}{n}\left( {\bar{\eta}T_{\rm avg}} + {\eta_l^2T_{\max}^2}\right) \sum_{i=1}^m\sum_{\tau=1}^{T_i} \sum_{j=1}^{n_i} \left( \widehat{y}_{i,j | {\bf D}_{i,\tau-1}^{(r)}} - 1 \right). \numberthis \label{eqn:orp2}
\end{align*}
Now, using Lemma \ref{lem:linsep}-(iii), $\| \bm{\Delta} \mathbf{W}^{(r)} \|_F^2$ can be bounded as
\begin{align*}
& \| \bm{\Delta} \mathbf{W}^{(r)} \|_F^2 \hspace{-.1cm} \leq \hspace{-.1cm} \left( \frac{\bar{\eta}^2}{n}\sum_{i=1}^m\sum_{\tau=1}^{T_i} n_{i} \right) \sum_{i=1}^m\sum_{\tau=1}^{T_i} \frac{n_i}{n} \|\nabla\mathcal{L}_i(\mathbf{D}^{(r)}_{i,\tau-1})\|^2_F\\
&\leq \frac{2\bar{\eta}^2T_{\rm avg}}{n} \sum_{i=1}^m\sum_{\tau=1}^{T_i}\sum_{j=1}^{n_i} \left(1-\widehat{y}_{i,j | \mathbf{D}^{(r)}_{i,\tau-1}} \right).\numberthis \label{eqn:orp3}
\end{align*}
Using Lemma \ref{lem:linsep}, the residual term ${\cal E}^{(r)}$ in \eqref{eq:cal_e_r} satisfies: 
\begin{lemma} \label{lem:cal_e_r}
    Under Assumption \ref{ass:linsep_dist}, for any $r \geq 0$, it holds
    \begin{equation} \label{eq:bound_cal_e_r}
    {\cal E}^{(r)} \leq \frac{2\bar{\eta}\eta_l^2 T_{\max}^2}{n}\sum_{i=1}^m  \sum_{t=1}^{T_i} \sum_{j=1}^{n_i} \left(1-\widehat{y}_{i,j | {\bf D}^{(r)}_{i,t-1}} \right).
    \end{equation}
\end{lemma}

Combining \eqref{eq:bound_cal_e_r}, \eqref{eqn:orp2}, and \eqref{eqn:orp3}, we obtain ${\cal A}_2 + {\cal E}^{(r)} + \| \bm{\Delta} \mathbf{W}^{(r)} \|_F^2 \leq 0$. As such, 
\begin{align*}
& \| {\bf W}^{(r+1)} - {\bf U} \|_F^2 - \| {\bf W}^{(r)} - {\bf U} \|_F^2 \\
& \leq \frac{2}{\gamma^2 R} - 2 \bar{\eta} T_{\min} {\cal L}( {\bf W}^{(r)} ). \numberthis \label{eq:recur_final}
\end{align*}
Re-arranging terms leads to 
\begin{align*}
\frac{1}{R} \sum_{r=0}^{R-1} {\cal L}( {\bf W}^{(r)} ) \leq \frac{1}{2 \bar{\eta} T_{\min} R} \left( \| {\bf W}^{(0)} - {\bf U} \|_F^2 + \frac{2}{\gamma^2}\right). 
\end{align*}
Using that
\begin{align}
    \| {\bf U} \|_F^2 \leq \frac{\log^2(\gamma^2\bar{\eta} R(K-1)T_{\rm avg})+\bar{\eta}^2T_{\rm avg}^2+\eta_l^4 T_{\max}^4}{\gamma^2 / 2} ,\label{eq:bound_u}
\end{align}
we obtain \eqref{eq:risk_bd} as a result.
Iterating from \eqref{eq:recur_final} yields
\begin{equation}
    \lVert\mathbf{W}^{(R)}-\mathbf{U}\rVert_F^2 \leq \lVert\mathbf{W}^{(0)}-\mathbf{U}\rVert_F^2+ {2}/ {\gamma^2}.
\end{equation}
Taking square root on both sides and applying triangle inequality yields
\begin{align*}
    \lVert \mathbf{W}^{(R)}\rVert_F & \leq \lVert\mathbf{W}^{(0)}\rVert_F+\lVert\mathbf{U}\rVert_F+ {\sqrt{2}} / {\gamma}. \numberthis 
\end{align*}
Applying \eqref{eq:bound_u} yields \eqref{eq:param_bound} and conclude our proof.

\begin{remark}
Compared to the proof of \cite[Theorem 4.1]{crawshaw2025constant}, our proof relies on a different choice of the reference point ${\bf U}$, together with a recursion for the round-by-round progress of {\tt FedAvg}. These non-trivial modifications have enabled us to obtain a tighter bound than \cite[Theorem 4.1]{crawshaw2025constant} that also handles the heterogeneous devices case.
\end{remark}

\subsection{Proof of Theorem \ref{thm:stablep}}
We preface the proof by defining the gradient potentials:
\begin{align*}
    \mathcal{G}(\mathbf{W}) & := \sum_{i=1}^m \frac{n_i}{n} \mathcal{G}_i(\mathbf{W}),~\mathcal{G}_i(\mathbf{W}) := \sum_{j=1}^{n_i} \frac{1-\widehat{y}_{i,j | {\bf W}} }{n_i}.
\end{align*}
Note that ${\cal G}_i( {\bf W} ) \leq {\cal L}_i( {\bf W} )$ for all $i$. 
We shall prove \eqref{eq:monotone_result} through showing by induction that for any $r \geq s$, it holds
\begin{subequations} \label{eq:induction_claim}
\begin{align}
    & \mathcal{L}(\mathbf{W}^{(r)}) \leq \min \left\{ \frac{1}{2n}, \frac{1}{4 \bar{\eta} T_{\rm avg}},  \frac{\gamma T_{\rm avg}}{40n\eta_l  T_{\rm max}^2 \sqrt{K}} \right\}, \nonumber \\
    & \mathcal{G}(\mathbf{W}^{(r)}) \leq 1/(2n), \label{eq:induction_claim_a} \\
    & {\cal L} ({\bf W}^{(r+1)}) \leq {\cal L}( {\bf W}^{(r)}). \label{eq:induction_claim_b}
\end{align}
\end{subequations}

For the base case with $r=s$, we note that \eqref{eq:induction_claim_a} holds by the assumption of Theorem \ref{thm:stablep} and ${\cal G}({\bf W}^{(r)} ) \leq {\cal L}({\bf W}^{(r)})$. 
To establish \eqref{eq:induction_claim_b}, the key step is to analyze the one-step progress from ${\bf W}^{(r)}$ to ${\bf W}^{(r+1)}$ and to bound ${\cal L}( {\bf W}^{(r+1)} )$ via a descent lemma type argument on $\log {\cal L}( {\bf W}^{(r)} )$. 
To this end, the server-side update can be written as:
\[
\mathbf{W}^{(r+1)}-\mathbf{W}^{(r)}:= -\bar{\eta}T_{\rm avg}(\nabla\mathcal{L}(\mathbf{W}^{(r)}) + \mathbf{E}^{(r)} + \tilde{\mathbf{E}}^{(r)}),
\]
\[ \textstyle
\mathbf{E}^{(r)}= \sum_{i=1}^m \sum_{\tau=1}^{T_i} \frac{n_i}{nT_{\rm avg}} \left(\nabla\mathcal{L}_i(\mathbf{D}_{i,\tau-1}^{(r)})-\nabla\mathcal{L}_i (\mathbf{W}^{(r)})\right),  
\]
\[ \textstyle
\tilde{\mathbf{E}}^{(r)}= \sum_{i=1}^m \sum_{\tau=1}^{T_i} \frac{n_i}{nT_{\rm avg}} \left( \nabla\mathcal{L}_i (\mathbf{W}^{(r)}) - \nabla\mathcal{L} (\mathbf{W}^{(r)}) \right).
\]
In Appendix~\ref{app:stablep}, we show that under the assumptions of Theorem \ref{thm:stablep} and ${\cal G}({\bf W}^{(r)}) \leq 1/(2n)$, it holds that
\begin{equation}
    \max\{ \|{\mathbf{E}}^{(r)}\|_F, \| \tilde{\mathbf{E}}^{(r)} \|_F \} \leq \frac{1}{10} \|\nabla\mathcal{L}(\mathbf{W}^{(r)})\|_F.
\end{equation}
The above bound yields 
\begin{equation} \label{eq:er_er_bound}
\|\nabla\mathcal{L}(\mathbf{W}^{(r)})+\mathbf{E}^{(r)} + \tilde{\bf E}^{(r)}\|_F \geq 4 \|{\mathbf{E}}^{(r)} + \tilde{\mathbf{E}}^{(r)} \|_F.
\end{equation}
Finally, using the $2$-smoothness of $\log \mathcal{L}(\mathbf{W})$ (cf.~Lemma~\ref{lem:lipcon}),
\begin{align*}
& \log \mathcal{L}(\mathbf{W}^{(r+1)})-\log \mathcal{L}(\mathbf{W}^{(r)}) \\
& \leq \left\langle\frac{\nabla\mathcal{L}(\mathbf{W}^{(r)})}{\mathcal{L}(\mathbf{W}^{(r)})}, \mathbf{W}^{(r+1)}-\mathbf{W}^{(r)}\right\rangle\\
& \quad+\|\mathbf{W}^{(r+1)}-\mathbf{W}^{(r)}\|_F^2 \\
& = -\frac{\bar{\eta} T_{\rm avg}}{\mathcal{L}(\mathbf{W}^{(r)})}\left\|\nabla\mathcal{L}(\mathbf{W}^{(r)})+\mathbf{E}^{(r)} + \tilde{\bf E}^{(r)} \right\|_F^2\\
& \quad+\frac{\bar{\eta} T_{\rm avg}}{\mathcal{L}(\mathbf{W}^{(r)})}\left\langle \mathbf{E}^{(r)} + \tilde{\bf E}^{(r)} , \nabla\mathcal{L}(\mathbf{W}^{(r)})+\mathbf{E}^{(r)} + \tilde{\bf E}^{(r)} \right\rangle \\
& \quad+\bar{\eta}^2 T_{\rm avg}^2\|\nabla\mathcal{L}(\mathbf{W}^{(r)})+\mathbf{E}^{(r)} + \tilde{\bf E}^{(r)} \|_F^2.\numberthis
\end{align*}
Subsequently, applying Cauchy-Schwarz and \eqref{eq:er_er_bound} yields
\begin{align*}
& \log \mathcal{L}(\mathbf{W}^{(r+1)})- \log \mathcal{L}(\mathbf{W}^{(r)}) \\ 
& \leq -\frac{\bar{\eta} T_{\rm avg}}{\mathcal{L}(\mathbf{W}^{(r)})}\left\|\nabla\mathcal{L}(\mathbf{W}^{(r)})+\mathbf{E}^{(r)} + \tilde{\bf E}^{(r)} \right\|_F^2\\
&\quad+\frac{\bar{\eta} T_{\rm avg}}{4\mathcal{L}(\mathbf{W}^{(r)})} \left\| \nabla\mathcal{L}(\mathbf{W}^{(r)})+\mathbf{E}^{(r)} + \tilde{\bf E}^{(r)} \right\|_F^2                                \\
& \quad+\bar{\eta}^2 T_{\rm avg}^2\| \nabla\mathcal{L}(\mathbf{W}^{(r)}) + \mathbf{E}^{(r)} + \tilde{\bf E}^{(r)} \|_F^2.\numberthis
\end{align*}%
As 
$\mathcal{L}(\mathbf{W}^{(s)}) \leq 1/(4\bar{\eta} T_{\rm avg})$, 
the right hand side of the above is upper bounded by $0$, implying $\log \mathcal{L}(\mathbf{W}^{(r+1)})\leq \log \mathcal{L}(\mathbf{W}^{(r)})$. 

For the induction step, assuming that \eqref{eq:induction_claim} holds up to certain $r \geq s$. Then, we note that \eqref{eq:induction_claim_a} for $r+1$ can be verified by ${\cal L}({\bf W}^{(r+1)}) \leq {\cal L}({\bf W}^{(r)})$ and ${\cal G}({\bf W}^{(r+1)}) \leq {\cal L}({\bf W}^{(r+1)})$.
Finally, the same argument for establishing \eqref{eq:induction_claim_b} at the base case can be applied to show that ${\cal L}({\bf W}^{(r+2)}) \leq {\cal L}({\bf W}^{(r+1)})$.
This concludes the proof of Theorem \ref{thm:stablep}.

\subsection{Proof of Corollary \ref{thm:stabled}}
We shall prove \eqref{eq:last_iter_risk} through showing that for any $r \geq s$,
\begin{equation} \label{eq:stabled_diff}
    \frac{1}{\mathcal{L}(\mathbf{W}^{(r+1)})}-\frac{1}{\mathcal{L}(\mathbf{W}^{(r)})}\geq \frac{\bar{\eta} T_{\rm avg} \gamma^2}{16}.
\end{equation}
We notice from the proof of Theorem \ref{thm:stablep} that \eqref{eq:induction_claim} holds for any $r \geq s$. As we shall show in Appendix \ref{app:stabled}, the above conditions guarantee that 
\begin{align*}
    \mathcal{L}(\mathbf{W}^{(r+1)}) & \leq {\cal L}( {\bf W}^{(r)} ) - \frac{ \bar{\eta} T_{\rm avg} \gamma^2 }{16} \left( {\cal L} ( {\bf W}^{(r)}) \right)^2.\numberthis \label{eq:stable_recur}
\end{align*}
{Dividing} both sides by $\mathcal{L}(\mathbf{W}^{(r)}) \mathcal{L}(\mathbf{W}^{(r+1)})$, we obtain 
\begin{align*}
    \frac{1}{\mathcal{L}(\mathbf{W}^{(r+1)})}-\frac{1}{\mathcal{L}(\mathbf{W}^{(r)})} & =\frac{\mathcal{L}(\mathbf{W}^{(r)})-\mathcal{L}(\mathbf{W}^{(r+1)})}{\mathcal{L}(\mathbf{W}^{(r)})\mathcal{L}(\mathbf{W}^{(r+1)})} \\
    & \geq \frac{ \bar{\eta} T_{\rm avg} \gamma^2 \mathcal{L}(\mathbf{W}^{(r)})}{16\mathcal{L}(\mathbf{W}^{(r+1)})}.\numberthis
\end{align*}
Since $\mathcal{L}(\mathbf{W}^{(r+1)})\leq \mathcal{L}(\mathbf{W}^{(r)})$, we obtain \eqref{eq:stabled_diff}. It remains to apply a telescoping argument which shows 
\[
\frac{1}{\mathcal{L}(\mathbf{W}^{(r)})}\geq\frac{1}{\mathcal{L}(\mathbf{W}^{(s)})} +\frac{\bar{\eta} T_{\rm avg} \gamma^2}{16}(r-s).\numberthis
\]
Using $\mathcal{L}(\mathbf{W}^{(s)})> 0$ proves \eqref{eq:last_iter_risk} and concludes our proof. 
 
\section{Numerical Experiments}\label{sec:exp}
This section reports the numerical experiments on a multi-class classification task using the large stepsize {\tt FedAvg} algorithm. We consider a synthetic dataset with an input dimension $d=12$, $K=5$ classes, $n=5,000$ samples. Each sample in the global dataset ${\cal D} = {\cal D}_1 \cup \cdots \cup {\cal D}_m$ is constructed via a rejection sampling mechanism such that the linear classifier defined by the weight matrix $\mathbf{W}_\star\in\mathbb{R}^{K\times d}$ satisfies Assumptions \ref{ass:bddf_dist} and \ref{ass:linsep_dist} with $\gamma = 0.001$. Note that under this setup, as ${\cal L}({\bf W})$ is $1/2$-smooth, traditional convergence analysis of {\tt FedAvg} requires a small stepsize as $\bar{\eta} = \eta_l \eta_g \leq 2$.
We use a label‑imbalanced data split: each client receives samples from at most two distinct classes.

\subsection{Effect of Heterogeneous Local Epochs}
We first examine the effect of heterogeneous devices on the convergence behavior. To this end, the dataset is partitioned among $m=8$ clients; each client receives exactly $625$ samples.
In this experiment, we fix the local stepsize at $\eta_l=1$, while the global stepsizes $\eta_g\in\{64, 128\}$ are examined. Two configurations for the number of local update steps, $T_i$, are considered: (i) {\it heterogeneous} local steps, where clients $1$-$4$ perform $T_i=4$ local updates per communication round, while clients $5$-$8$ perform $T_i=50$ updates, yielding $T_\text{avg}=27$; (ii) {\it homogeneous} local steps, where all clients perform $T_i=27$ local steps per round. 

\begin{figure}[t]
\includegraphics[width=1.0\linewidth]{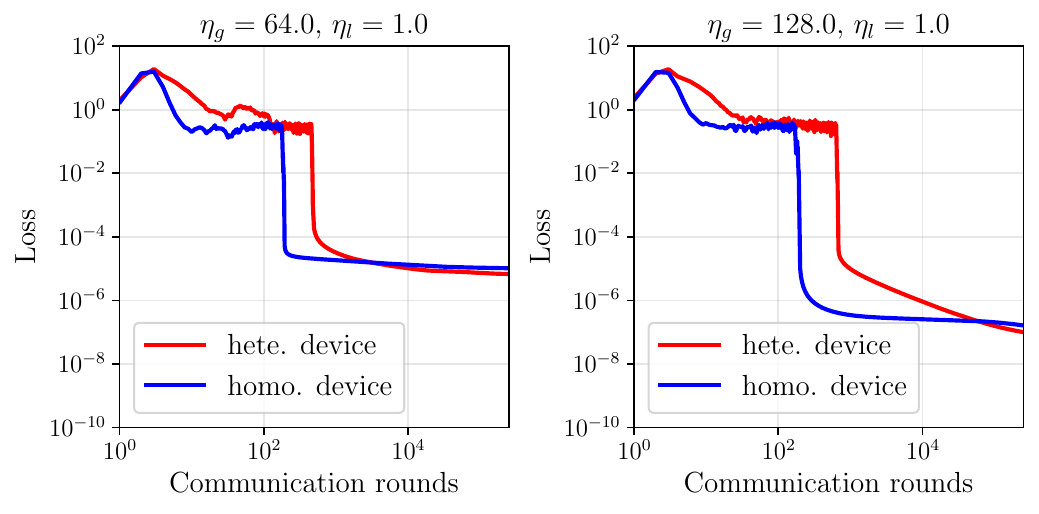}
\caption{Convergence of {\tt FedAvg} with homogeneous (``homo. device'': $T_i=27$ for all devices) and heterogeneous (``hete. device'': $T_i=4$ for device $1$-$4$, $T_i=50$ for device $5$-$8$) devices. Both configurations share the same average $T_\text{avg}=27$ with global/local stepsizes: (Left) $\eta_g=64, \eta_l=1$, (Right) $\eta_g=128,\eta_l=1$.} \vspace{-.2cm} \label{fig:homohetero}
\end{figure}
The {\tt FedAvg} algorithm is executed for $R=250,000$ communication rounds. Fig. \ref{fig:homohetero} shows the trajectory of cross‑entropy loss on the full dataset for each communication round. As observed, in both global stepsize configurations, the heterogeneous device case exhibits a longer transient with larger fluctuations than the homogeneous device case before entering a stable monotonic regime. This behavior is consistent with Theorem \ref{thm:eos_average_loss_bound} which predicts a slower convergence for heterogeneous devices. 
However, we observe that after the transient phase as $R \gg 1$, the convergence behavior of both heterogeneous and homogeneous cases become indistinguishable. This corroborates with Corollary \ref{thm:stabled} which shows that the asymptotic rate is dominated by ${\cal O}( 1 / (T_{\rm avg} R))$, independent of the device heterogeneity\footnote{Note that the distribution of $\{ T_i \}_{i=1}^m$ may violates the condition required by Theorem \ref{thm:stablep}, yet we observe that the predicted behavior continues to hold.}.



\subsection{Effect of Large Local Steps under Label Imbalance}

\begin{figure}[!t]
    \centering
    \vspace{.25cm}
    \includegraphics[width=1.0\linewidth]{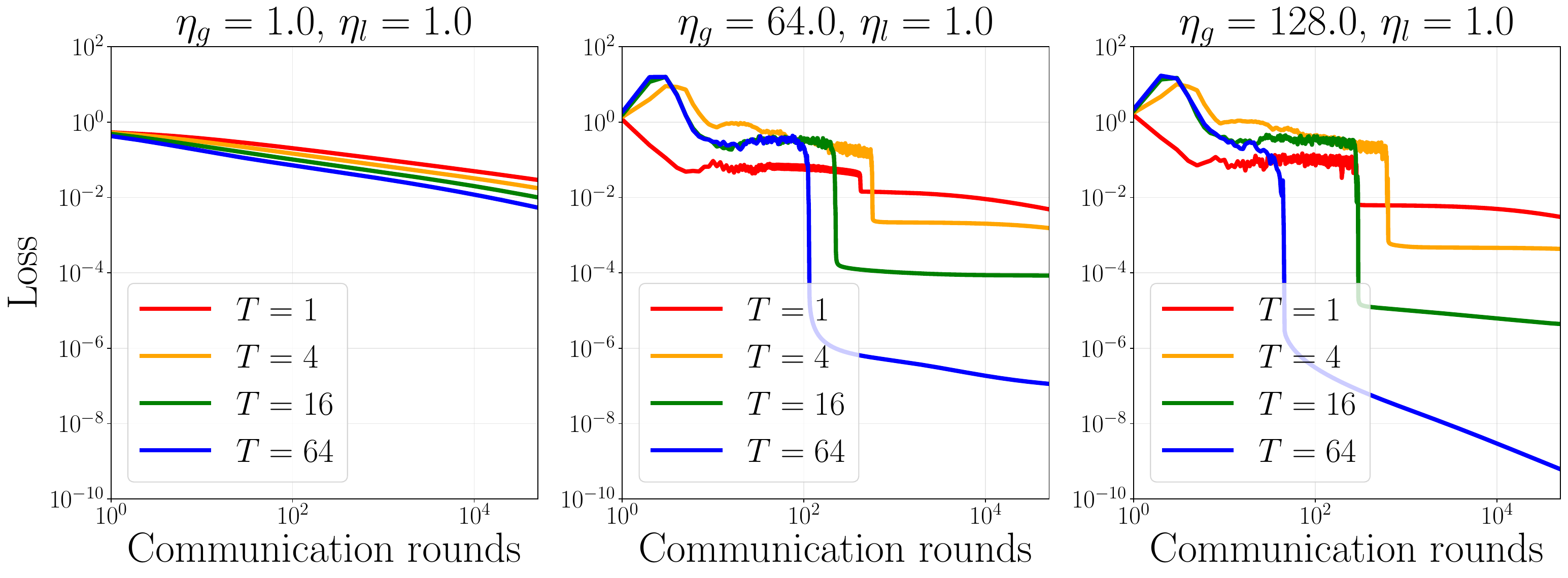}
    \caption{Effects of the number of local steps $T$. All devices perform the same number of local epochs $T \in \{1,4,16,64\}$, with global stepsizes $\eta_g\in\{1, 64, 128\}$ and local stepsize $\eta_l=1$.}\vspace{-.2cm}
    \label{fig:sameT}
\end{figure}

We next examine the effect of increasing the common number of local steps $T = T_i$ for each $i$. All clients perform the same number of local epochs $T \in \{1,4,16,64\}$, with the stepsizes as $\eta_g\in\{1, 64, 128\}$ and $\eta_l=1$.

The {\tt FedAvg} algorithm is executed for $R=50,000$ communication rounds. Fig.~\ref{fig:sameT} plots the cross-entropy loss value against communication rounds with varying $T$. We first observe that for sufficiently large $R$, the convergence rates of all configurations are close to ${\cal O}(1/(T R))$ as predicted by Corollary \ref{thm:stabled}, where increasing $T$ leads to a faster convergence. Furthermore, we note that the large stepsizes settings $\eta_g \in \{64, 128\}$ significantly accelerate the convergence compared to the baseline case with $\eta_g=1$, where the convergence is significantly slower.

\section{Conclusion}\label{sec:conclusion}
This paper has extended the analysis of large stepsize {\tt FedAvg} to the case of heterogeneous devices under a linear separability condition. We showed that the effect of device heterogeneity vanishes asymptotically, where the convergence rate is dominated by the average number of local steps across devices and validated the findings by numerical experiments. For future work, it will be interesting to explore non-separable data and regularized losses. 

\bibliographystyle{ieeetr}
\bibliography{cdc26.bib}

\newpage
\appendices

\section{Proof of Lemma \ref{lem:linsep}}
To prove \eqref{eq:lij_bound}, notice that for any $k \neq y_{i,j}$, Assumption \ref{ass:linsep_dist} guarantees that 
$( \mathbf{w}_{\star,y_{i,j}}- \mathbf{w}_{\star,k})^\top\mathbf{x}_{i,j}\geq \gamma$.
Therefore,  
\begin{align*}
& \ell_{i,j}( \alpha {\bf W}_\star )\numberthis \\
& \textstyle  =\log \Big( 1 + \sum_{k\neq y_{i,j}} \exp( \alpha \mathbf{w}_{\star,k}^\top\mathbf{x}_{i,j}-\alpha \mathbf{w}_{\star,y_{i,j}}^\top\mathbf{x}_{i,j}) \Big)\\
& \leq \log\Big(1+ (K-1) \exp\big( \alpha \max_{k\neq y_{i,j}}(\mathbf{w}_{\star,k}-\mathbf{w}_{\star,y_{i,j}})^\top\mathbf{x}_{i,j}\big)\Big)\\
& \leq \log\left(1+ (K-1) \exp\left( -\alpha \gamma \right) \right) \leq (K-1)\exp\left( -\alpha \gamma \right).
\end{align*}
To prove \eqref{eq:lij_gradient_match}, we first notice that the inner product is
\begin{align*}
    & \langle \nabla \ell_{i,j} ( {\bf W} ), {\bf W}_\star \rangle = \sum_{k=1}^K \left( \widehat{y}_{i,j | {\bf W}}^{(k)} - \mathds{1}(y_{i,j}=k) \right) \mathbf{x}_{i,j}^\top \mathbf{w}_{\star,k},
\end{align*}
where $\mathds{1} (\cdot)$ is the indicator function with $\mathds{1}(y_{i,j}=k) = 1$ for $k = y_{i,j}$ and $0$ otherwise. Under Assumption \ref{ass:linsep_dist}, we have 
\begin{align*}
& \textstyle \sum_{k=1}^K \left( \widehat{y}_{i,j | {\bf W}}^{(k)} - \mathds{1}(y_{i,j}=k) \right) \mathbf{x}_{i,j}^\top \mathbf{w}_{\star,k} \\
& \textstyle = (\widehat{y}_{i,j | {\bf W}} - 1)  \mathbf{x}_{i,j}^\top \mathbf{w}_{\star,y_{i,j}} + \sum_{k \neq y_{i,j}} \widehat{y}_{i,j | {\bf W}}^{(k)} \mathbf{x}_{i,j}^\top \mathbf{w}_{\star,k} \\
& \textstyle \leq (\widehat{y}_{i,j | {\bf W}} - 1)  \mathbf{x}_{i,j}^\top \mathbf{w}_{\star,y_{i,j}} \hspace{-.15cm} + \sum_{k \neq y_{i,j}} \widehat{y}_{i,j | {\bf W}}^{(k)} \left(  \mathbf{x}_{i,j}^\top \mathbf{w}_{\star,y_{i,j}} \hspace{-.15cm} - \gamma \right) \\
& = \textstyle \left( \sum_{k=1}^K \widehat{y}_{i,j | {\bf W}}^{(k)} - 1 \right) \mathbf{x}_{i,j}^\top \mathbf{w}_{\star,y_{i,j}} - \gamma \sum_{k \neq y_{i,j}} \widehat{y}_{i,j | {\bf W}}^{(k)} \\
& =-\gamma (1- \widehat{y}_{i,j | {\bf W}} ).\numberthis
\end{align*} 
To prove \eqref{eq:lij_gradient_bound}, we observe that 
\begin{align*}
    & \lVert \nabla\ell_{i,j}(\mathbf{W})\rVert_F^2 = \textstyle \sum_{k=1}^K \left\| \left( \widehat{y}_{i,j | {\bf W}}^{(k)} - \mathds{1}(y_{i,j}=k) \right) {\bf x}_{i,j} \right\|^2                                                           \\
    & \textstyle =\lVert \mathbf{x}_{i,j}\rVert^2 \Big((1-\widehat{y}_{i,j | {\bf W}})^2+\sum_{k\neq y_{i,j}} (\widehat{y}_{i,j | {\bf W}}^{(k)})^2\Big).\numberthis
\end{align*}
Since $(1- \widehat{y}_{i,j | {\bf W}})^2\leq 1-\widehat{y}_{i,j | {\bf W}}$
and
\begin{align*} \textstyle
    \sum_{k\neq y_{i,j}} (\widehat{y}_{i,j | {\bf W}}^{(k)})^2\leq \sum_{k\neq y_{i,j}} \widehat{y}_{i,j | {\bf W}}^{(k)}=1-\widehat{y}_{i,j | {\bf W}},
\end{align*} we have 
\begin{equation}
    \lVert \nabla\ell_{i,j}(\mathbf{W})\rVert_F^2\leq 2(1-\widehat{y}_{i,j | {\bf W}}).
\end{equation}

\section{Proof of Lemma \ref{lem:cal_e_r}}
We concentrate on the following term for each $i$, 
\[ \textstyle
\sum_{\tau=1}^{T_i} n_i \|\mathbf{W}^{(r)}-\mathbf{D}^{(r)}_{i,\tau-1}\|^2_F.\numberthis
\]
If $T_i = 1$, then the above term is zero. For $T_i \geq 2$, from local update rule
$\mathbf{D}^{(r)}_{i,\tau}=\mathbf{D}^{(r)}_{i,\tau-1}-\eta_l\nabla\mathcal{L}_i(\mathbf{D}^{(r)}_{i,\tau-1})$,
we have
\begin{align*}
&\|\mathbf{D}^{(r)}_{i,\tau}-\mathbf{W}^{(r)}\|_F^2 =\|\mathbf{D}^{(r)}_{i,\tau-1}-\eta_l\nabla\mathcal{L}_i(\mathbf{D}^{(r)}_{i,\tau-1})-\mathbf{W}^{(r)}\|_F^2 \\
&\leq \frac{T_i}{T_i-1}\|\mathbf{D}^{(r)}_{i,\tau-1}-\mathbf{W}^{(r)}\|_F^2 +\eta_l^2T_i\|\nabla \mathcal{L}_i(\mathbf{D}^{(r)}_{i,\tau-1})\|_F^2 \\
& \textstyle \leq \eta_l^2 T_i\sum_{t=1}^{\tau-1} \|\nabla \mathcal{L}_i(\mathbf{D}^{(r)}_{i,t-1})\|_F^2 \left( \frac{T_i}{T_i-1}\right)^{\tau-t}.\numberthis
\end{align*} 
We now proceed to bound
\begin{align*}
& {\cal E}^{(r)} = \frac{\bar{\eta}}{n}\sum_{i=1}^m \sum_{\tau = 2}^{T_i}n_i\|\mathbf{D}^{(r)}_{i,\tau-1}-\mathbf{W}^{(r)}\|_F^2\\&\leq\frac{\bar{\eta}\eta_l^2 }{n}\sum_{i=1}^m  n_i T_i \sum_{\tau=2}^{T_i} \sum_{t=1}^{\tau-1} \|\nabla\mathcal{L}_i(\mathbf{D}^{(r)}_{i,t-1})\|^2_F\left(1+\frac{1}{T_i-1}\right)^{\tau-t} \\
&=\frac{\bar{\eta}\eta_l^2}{n}\sum_{i=1}^m n_i T_i \sum_{t=1}^{T_i - 1} \sum_{\tau=1}^{T_i-t} \|\nabla\mathcal{L}_i(\mathbf{D}^{(r)}_{i,t-1})\|^2_F  \left(1+\frac{1}{T_i-1}\right)^{\tau}. 
\end{align*}
The above can be further bounded by:
\begin{align*}
&\leq\frac{\bar{\eta}\eta_l^2}{n}\sum_{i=1}^m n_i T_i \sum\limits_{t=1}^{T_i - 1} \sum_{\tau=1}^{T_i-1} \|\nabla\mathcal{L}_i(\mathbf{D}^{(r)}_{i,t-1})\|^2_F  \left(1+\frac{1}{T_i-1}\right)^{\tau} \\
&\leq\frac{\bar{\eta}\eta_l^2}{n}\sum\limits_{i=1}^m n_i T_i \sum\limits_{t=1}^{T_i-1} \|\nabla\mathcal{L}_i(\mathbf{D}^{(r)}_{i,t-1})\|^2_F\cdot (e-1)T_i \\
&\leq\frac{2\bar{\eta}\eta_l^2}{n}\sum_{i=1}^m n_i T_i^2 \sum_{t=1}^{T_i-1} \|\nabla\mathcal{L}_i(\mathbf{D}^{(r)}_{i,t-1})\|^2_F\\
&\leq\frac{4\bar{\eta}\eta_l^2}{n}\sum_{i=1}^m T_i^2 \sum_{t=1}^{T_i-1} \sum_{j=1}^{n_i} \left(1-\widehat{y}_{i,j | {\bf D}^{(r)}_{i,t-1}} \right). \numberthis
\end{align*}
 
\section{Missing Proofs of Theorem \ref{thm:stablep}} \label{app:stablep}
We preface the proof by a few additional lemmas.
\begin{lemma}\label{lem:lipcon}
Under Assumption \ref{ass:bddf_dist}. For every $i \in [m]$, the function $\log {\cal L}_i(\mathbf{W})$ is $2$-smooth.
\end{lemma}
\begin{lemma} 
    \label{thm:stablep_client}
    Under Assumptions \ref{ass:bddf_dist} and \ref{ass:linsep_dist}. Suppose that there exists $s \geq 0$ with $\mathcal{L}_i(\mathbf{D}^{(r)}_{i,s}) \leq 1/\eta_l$, then for every $\eta_l>0$,
    \[
    \mathcal{L}_i(\mathbf{D}^{(r)}_{i,\tau+1}) \leq \mathcal{L}_i(\mathbf{D}^{(r)}_{i,\tau}),~\forall~\tau \geq s,~\forall~i\in[m].\numberthis
    \] 
\end{lemma}
\begin{lemma}
    \label{lmm:selfbound}
    Under Assumption \ref{ass:bddf_dist}, for any $\mathbf{W}$, we have \[
        \lVert \nabla \mathcal{L}(\mathbf{W})\rVert_F\leq \sqrt{K} \mathcal{L}(\mathbf{W}), ~~
        \lVert \nabla^2 \mathcal{L}(\mathbf{W})\rVert_\emph{op}\leq 2\mathcal{L}(\mathbf{W}).\numberthis
    \]
\end{lemma}
\begin{lemma} \label{lmm:gradpot} 
    Under Assumptions \ref{ass:bddf_dist} \&
    \ref{ass:linsep_dist}, for every $\eta_g, \eta_l>0$,
    \begin{align*}
    & \frac{1}{R}\sum_{r=0}^{R-1} \mathcal{G}(\mathbf{W}^{(r)}) \leq \frac{\sqrt{2} + \| \mathbf{W}^{(0)} \|_F }{ \bar{\eta} \gamma^2 R }\\
    & + \frac{ \log(\gamma^2\bar{\eta} R(K-1)T_{\rm avg})+\bar{\eta}T_{\rm avg}+{\eta}_l^2 T_{\max}^2 }{\bar{\eta}\gamma^2 R / 2}. \numberthis \label{eq:gradpot_bound}
    \end{align*}
    Furthermore, if $\mathcal{G}(\mathbf{W})\leq 1/(2n)$, then 
    \[
    \|\nabla\mathcal{L}(\mathbf{W})\|_F\geq (\gamma/2)\mathcal{L}(\mathbf{W}).
    \]
\end{lemma}

\noindent
\textbf{Proof of \eqref{eq:er_er_bound}.} Our main focus is to prove \eqref{eq:er_er_bound}.
Let $C= \{(1-t)\mathbf{W}^{(r)}+t\mathbf{D}^{(r)}_{i,\tau-1}:t\in(0,1) \}$,  notice that 
\begin{align*}
    & \|\mathbf{E}^{(r)}\|_F \leq \sum_{i=1}^m \sum_{\tau=1}^{T_i} \frac{n_i}{nT_{\rm avg}} \|\nabla\mathcal{L}_i(\mathbf{D}_{i,\tau-1}^{(r)})-\nabla\mathcal{L}_i(\mathbf{W}^{(r)})\|_F \\
    & {\leq} \sum_{i=1}^m \sum_{\tau=1}^{T_i} \frac{n_i}{nT_{\rm avg}} \max_{\mathbf{W}\in C} \|\nabla^2\mathcal{L}_i(\mathbf{W})\|_{\rm op} \left\|\mathbf{D}_{i,\tau-1}^{(r)}-\mathbf{W}^{(r)}\right\|_F \\
    & {\leq} \sum_{i=1}^m \sum_{\tau=1}^{T_i} \frac{2 n_i}{nT_{\rm avg}} \max_{\mathbf{W}\in C} \mathcal{L}_i(\mathbf{W}) \left\|\mathbf{D}_{i,\tau-1}^{(r)}-\mathbf{W}^{(r)}\right\|_F,
\end{align*}
where the second last inequality is due to Lemma \ref{lmm:selfbound}.
Using Lemma \ref{thm:stablep_client}, $\mathcal{L}_i(\mathbf{D}_{i,\tau-1}^{(r)})\leq \mathcal{L}_i(\mathbf{W}^{(r)})$ for all $\tau\geq 1$, we have
\begin{align*}
    & \|\mathbf{E}^{(r)}\|_F \leq \frac{2}{nT_{\rm avg}}\sum_{i=1}^m \sum_{\tau=1}^{T_i} n_i \mathcal{L}_i(\mathbf{W}^{(r)})\left\|\mathbf{D}_{i,\tau-1}^{(r)}-\mathbf{W}^{(r)}\right\|_F \\
    & \leq \frac{2}{nT_{\rm avg}}\sum_{i=1}^m \sum_{\tau=1}^{T_i-1} n_i \mathcal{L}_i(\mathbf{W}^{(r)})\left\|\sum\limits_{t=1}^{\tau}\eta_l\nabla\mathcal{L}_i(\mathbf{D}_{i,t-1}^{(r)})\right\|_F  \\
    & \leq  \frac{2\eta_l}{nT_{\rm avg}} \sum_{i=1}^m\sum_{\tau=1}^{T_i-1} \sum\limits_{t=1}^{\tau}n_i \mathcal{L}_i(\mathbf{W}^{(r)})\left\|\nabla\mathcal{L}_i(\mathbf{D}_{i,t-1}^{(r)})\right\|_F.\numberthis
\end{align*}
Lemma \ref{lmm:selfbound} shows $\|\nabla\mathcal{L}_i(\mathbf{D}_{i,t-1}^{(r)})\|_F \leq \sqrt{K}\mathcal{L}_i(\mathbf{D}_{i,t-1}^{(r)})$, and hence we can bound $\| {\bf E}^{(r)} \|_F$ by
\begin{align*}
    & \textstyle \frac{2\eta_l}{nT_{\rm avg}} \sum_{i=1}^m \sum_{\tau=1}^{T_i-1} \sum\limits_{t=1}^{\tau}n_i \mathcal{L}_i(\mathbf{W}^{(r)})(\sqrt{K}\mathcal{L}_i(\mathbf{D}_{i,t-1}^{(r)}))\\
    & \textstyle \leq  \frac{2\eta_l\sqrt{K}}{nT_{\rm avg}} \sum_{i=1}^m (n_iT_i\mathcal{L}_i(\mathbf{W}^{(r)}))^2\\
    & \textstyle \leq  \frac{2\eta_l\sqrt{K}T_{\rm max}^2}{nT_{\rm avg}} \sum_{i=1}^m (n_i \mathcal{L}_i(\mathbf{W}^{(r)}))^2  \\
    & \leq 
    \frac{2n\eta_l\sqrt{K}T_{\rm max}^2}{T_{\rm avg}} \left(\mathcal{L}(\mathbf{W}^{(r)})\right)^2.\numberthis
\end{align*}
Using $\mathcal{L}(\mathbf{W}^{(r)})\leq \gamma T_{\rm avg} /(40 n\eta_l T_{\max}^2 \sqrt{K})$ and  Lemma \ref{lmm:gradpot}, we have that
$\|\mathbf{E}^{(r)}\|_F \leq \frac{\gamma}{20} \mathcal{L}(\mathbf{W}^{(r)}) \leq \frac{1}{10}\|\nabla\mathcal{L}(\mathbf{W}^{(r)})\|_F$.

Next, we bound $\|\tilde{\mathbf{E}}^{(r)}\|_F$. Observe that 
\begin{align*}
    & \tilde{\bf E}^{(r)} = \frac{1}{nT_{\rm avg}}\sum_{i=1}^m n_i T_i \cdot\left( \nabla\mathcal{L}_i (\mathbf{W}^{(r)}) - \nabla\mathcal{L} (\mathbf{W}^{(r)}) \right) \\
    &= \frac{1}{nT_{\rm avg}}\sum_{i=1}^m n_i T_i \nabla\mathcal{L}_i (\mathbf{W}^{(r)}) - \nabla {\cal L} (\mathbf{W}^{(r)})  \\
    &= \frac{1}{nT_{\rm avg}}\sum_{i=1}^m n_i T_i \nabla\mathcal{L}_i (\mathbf{W}^{(r)}) - \frac{1}{n}\sum_{i=1}^m n_i \nabla\mathcal{L}_i (\mathbf{W}^{(r)}) \\
    &= \frac{1}{n}\sum_{i=1}^m n_i \left(\frac{T_i}{T_{\rm avg}}-1\right)\nabla\mathcal{L}_i (\mathbf{W}^{(r)}). \numberthis
\end{align*}
By Lemma \ref{lmm:selfbound}, $\| \nabla\mathcal{L}_i (\mathbf{W}^{(r)}) \|_F \leq \sqrt{K}\mathcal{L}_i(\mathbf{W}^{(r)})$. Further using the last inequality in Lemma \ref{lmm:gradpot}, we have 
\begin{align*}
    \|\tilde{\bf E}^{(r)}\|_F & \textstyle \leq \frac{1}{n}\sum_{i=1}^m n_i \left|\frac{T_i}{T_{\rm avg}}-1\right|\cdot \| \nabla\mathcal{L}_i (\mathbf{W}^{(r)}) \|_F \\
    & \textstyle \leq \frac{\sqrt{K}}{n}\sum_{i=1}^m n_i \left|\frac{T_i}{T_{\rm avg}}-1\right|\cdot \mathcal{L}_i(\mathbf{W}^{(r)}) \\
    & \textstyle \leq \max_i \left|\frac{T_i}{T_{\rm avg}}-1\right| \, \sqrt{K} \frac{2}{\gamma} \|\nabla\mathcal{L}(\mathbf{W}^{(r)})\|_F. \numberthis
\end{align*}
Under the assumption that $\max_i |T_i/T_{\rm avg}-1| \leq \gamma/(20\sqrt{K})$, we have that $\|\tilde{\bf E}^{(r)}\|_F \leq \frac{1}{10}\|\nabla\mathcal{L}(\mathbf{W}^{(r)})\|_F$.

\subsection{Proof of Lemma \ref{lem:lipcon}}
To verify that $\log \mathcal{L}_i(\mathbf{W})$ is 2-smooth, we observe that:
\begin{align*}
    & \nabla^2 \log \mathcal{L}_i(\mathbf{W}) \\
    & =\frac{1}{\mathcal{L}_i(\mathbf{W})^2}\biggl(\nabla^2\mathcal{L}_i(\mathbf{W})\cdot \mathcal{L}_i(\mathbf{W}) - (\nabla\mathcal{L}_i(\mathbf{W}))^{\otimes 2}\biggr) \\
    & \preceq \frac{1}{\mathcal{L}_i(\mathbf{W})}\nabla^2\mathcal{L}_i(\mathbf{W})  \preceq \frac{1}{\mathcal{L}_i(\mathbf{W})} \mathcal{L}_i(\mathbf{W})\cdot 2\mathbf{I}=2\mathbf{I},\numberthis
\end{align*}
where the second last inequality comes from Lemma \ref{lmm:selfbound}. 

\subsection{Proof of Lemma \ref{thm:stablep_client}}
Using the $2$-smoothness of $\log \mathcal{L}_i(\cdot)$, we get \begin{align*}
& \log \mathcal{L}_i(\mathbf{D}^{(r)}_{i,\tau+1}) \numberthis\\
& \textstyle \leq \log \mathcal{L}_i(\mathbf{D}^{(r)}_{i,\tau})+\left\langle\frac{\nabla \mathcal{L}_i(\mathbf{D}^{(r)}_{i,\tau})}{\mathcal{L}_i(\mathbf{D}^{(r)}_{i,\tau})}, \mathbf{D}^{(r)}_{i,\tau+1}-\mathbf{D}^{(r)}_{i,\tau}\right\rangle \\
&\quad+ \lVert \mathbf{D}^{(r)}_{i,\tau+1}-\mathbf{D}^{(r)}_{i,\tau}\rVert^2_F\\
& \textstyle =\log \mathcal{L}_i(\mathbf{D}^{(r)}_{i,\tau})-\eta_l\left(\frac{1}{\mathcal{L}_i(\mathbf{D}^{(r)}_{i,\tau})}-\eta_l\right)\lVert\nabla\mathcal{L}_i(\mathbf{D}^{(r)}_{i,\tau})\rVert^2_F.\
\end{align*}
The above implies that if $\mathcal{L}(\mathbf{D}^{(r)}_{i,\tau})\leq \frac{1}{\eta_l}$, then 
\begin{equation}
    \mathcal{L}(\mathbf{D}^{(r)}_{i,\tau+1})\leq \mathcal{L}(\mathbf{D}^{(r)}_{i,\tau}).
\end{equation}

\section{Missing Proof of Corollary \ref{thm:stabled}} \label{app:stabled}
We shall prove \eqref{eq:stable_recur} in this section. By letting $\mathbf{V}=\lambda\mathbf{W}^{(r)}+(1-\lambda)\mathbf{W}^{(r+1)}, \lambda\in(0,1)$, we have 
\begin{align*}
    & \mathcal{L}(\mathbf{W}^{(r+1)}) \leq \mathcal{L}(\mathbf{W}^{(r)}) \numberthis\\
    & - \bar{\eta}T_{\rm avg}\lVert \nabla \mathcal{L}(\mathbf{W}^{(r)}) \rVert_F^2 -\bar{\eta}T_{\rm avg} \langle {\bf E}^{(r)} + \tilde{\bf E}^{(r)}, \nabla \mathcal{L}(\mathbf{W}^{(r)}) \rangle \\
    & +\frac{ (\bar{\eta}T_{\rm avg})^2 }{2}\lVert\nabla^2\mathcal{L}(\mathbf{\mathbf{V}})\rVert_{\text{op}} \lVert \nabla\mathcal{L}(\mathbf{W}^{(r)}) + {\bf E}^{(r)} + \tilde{\bf E}^{(r)} \rVert_F^2. 
\end{align*}
Using Lemma \ref{lmm:selfbound}, we have $\lVert\nabla^2\mathcal{L}(\mathbf{\mathbf{V}})\rVert_{\text{op}} \leq 2 {\cal L}(\mathbf{V}) \leq 2(\lambda {\cal L}(\mathbf{W}^{(r)})+(1-\lambda)\mathcal{L}(\mathbf{W}^{(r+1)}))\leq 2\mathcal{L}(\mathbf{W}^{(r)})$, where we again used $\mathcal{L}(\mathbf{W}^{(r+1)}) \leq \mathcal{L}(\mathbf{W}^{(r)})$. Thus,
\begin{align*}
    & \mathcal{L}(\mathbf{W}^{(r+1)}) \leq {\cal L}( {\bf W}^{(r)} ) - \bar{\eta}T_{\rm avg}\lVert \nabla \mathcal{L}(\mathbf{W}^{(r)}) \rVert_F^2  \\
    &   -\bar{\eta}T_{\rm avg} \langle {\bf E}^{(r)} + \tilde{\bf E}^{(r)}, \nabla \mathcal{L}(\mathbf{W}^{(r)}) \rangle \\
    & + 2 (\bar{\eta}T_{\rm avg})^2 \mathcal{L}(\mathbf{W}^{(r)}) \lVert \nabla\mathcal{L}(\mathbf{W}^{(r)}) \rVert_F^2 \\
    & + 2 (\bar{\eta}T_{\rm avg})^2 \mathcal{L}(\mathbf{W}^{(r)}) \lVert {\bf E}^{(r)} + \tilde{\bf E}^{(r)} \rVert_F^2. \numberthis
\end{align*}
Note we have $2 (\bar{\eta}T_{\rm avg}) \mathcal{L}(\mathbf{W}^{(r)}) \leq 1/2$ by the hypothesis in this corollary. It yields
\begin{align*}
    & \mathcal{L}(\mathbf{W}^{(r+1)}) \leq {\cal L}( {\bf W}^{(r)} ) - ( \bar{\eta}T_{\rm avg} / 2 ) \lVert \nabla \mathcal{L}(\mathbf{W}^{(r)}) \rVert_F^2  \\
    &   -\bar{\eta}T_{\rm avg} \langle {\bf E}^{(r)} + \tilde{\bf E}^{(r)}, \nabla \mathcal{L}(\mathbf{W}^{(r)}) \rangle \\
    & + 2 (\bar{\eta}T_{\rm avg})^2 \mathcal{L}(\mathbf{W}^{(r)}) \lVert {\bf E}^{(r)} + \tilde{\bf E}^{(r)} \rVert_F^2. \numberthis
\end{align*}
By Cauchy-Schwarz inequality and Lemma \ref{lmm:selfbound}, we further notice that 
\begin{align*}
& -\bar{\eta}T_{\rm avg} \langle {\bf E}^{(r)} + \tilde{\bf E}^{(r)}, \nabla \mathcal{L}(\mathbf{W}^{(r)}) \rangle \\
& + 2 (\bar{\eta}T_{\rm avg})^2 \mathcal{L}(\mathbf{W}^{(r)}) \lVert {\bf E}^{(r)} + \tilde{\bf E}^{(r)} \rVert_F^2 \\
& \leq \frac{ \bar{\eta} T_{\rm avg} }{5} \| \nabla {\cal L}( {\bf W}^{(r)} )  \|_F^2 + \frac{ \bar{\eta}T_{\rm avg}}{50} \| \nabla {\cal L}( {\bf W}^{(r)} )  \|_F^2.\numberthis
\end{align*}
Combining terms yields
\[
\mathcal{L}(\mathbf{W}^{(r+1)}) \leq {\cal L}( {\bf W}^{(r)} ) - \frac{ \bar{\eta} T_{\rm avg} }{4} \| {\cal L} ( {\bf W}^{(r)}) \|_F^2.\numberthis
\]
Further using $\|\nabla\mathcal{L}(\mathbf{W}^{(r)})\|_F\geq \frac{\gamma}{2}\mathcal{L}(\mathbf{W}^{(r)})$ (cf.~Lemma~\ref{lmm:gradpot}), we conclude the proof of \eqref{eq:stable_recur}.

\section{Auxiliary Lemmas}

\begin{lemma} \label{lmm:smoothloc}
    Under Assumptions \ref{ass:bddf_dist}, \ref{ass:linsep_dist}, the loss function $\mathcal{L}(\mathbf{W})$ is $1/2$-smooth.
\end{lemma}
The above lemma can be verified directly by evaluating the eigenvalues of $\nabla^2 {\cal L}( {\bf W})$.

\begin{lemma}\label{lmm:lipcon}
If $F_{m}$ is twice continuously differentiable, then for any $\mathbf{u},\mathbf{v}$ in a convex set $C=\left\{(1-t)\mathbf{u}+t\mathbf{v}:t\in(0,1)\right\}$, we have
    \[
    \| \nabla F_{m}(\mathbf{u}) - \nabla F_{m}(\mathbf{v}) \|\leq\left( \underset{\mathbf{w} \in C}{\max} \| \nabla^{2}F_{m}(\mathbf{w}) \| \right) \| \mathbf{u} - \mathbf{v} \|.
    \]
\end{lemma}
The above lemma is a direct consequence of Taylor's theorem for vector-valued functions.
  
\section{Proofs of Lemmas in Theorem \ref{thm:stablep}} \label{app:pf_lemmas}

\subsection{Proof of Lemma \ref{lmm:selfbound}}
We first need to bound the gradient norm with relation to the loss. Observe that
\begin{align*}
& \textstyle \lVert \nabla \ell_{i,j} (\mathbf{W})\rVert^2_F =\sum_{k=1}^K \lVert \nabla _{\mathbf{w}_k}\ell_{i,j} (\mathbf{W})\rVert^2 \\
& \textstyle \leq \lVert \mathbf{x}_{i,j} \rVert^2\Big((1-\widehat{y}_{i,j | {\bf W} })^2+\sum_{k \neq y_{i,j}} (\widehat{y}_{i,j | {\bf W}}^{(k)})^2\Big) \\
& \leq K(1-\widehat{y}_{i,j | {\bf W} } )^2\|\mathbf{x}_{i,j}\|^2 \leq K(1-\widehat{y}_{i,j | {\bf W} } )^2.\numberthis
\end{align*}
Since $-\log z\geq 1-z$ for $z\in(0,1]$, we have 
\begin{equation}
    \ell_i(\mathbf{W})= - \log\widehat{y}_{i,j | {\bf W} }  \geq 1-\widehat{y}_{i,j | {\bf W} },
\end{equation} 
\begin{equation}
    \lVert \nabla \ell_i(\mathbf{W})\rVert_F\leq \sqrt{K}(1-\widehat{y}_{i,j | {\bf W} })\leq \sqrt{K}\ell_i(\mathbf{W}).
\end{equation}
Henceforth, 
\begin{align*}
& \textstyle \lVert \nabla\mathcal{L}(\mathbf{W})\rVert_F\leq \frac{1}{n}\sum_{i=1}^m \sum_{j=1}^{n_i} \lVert \nabla \ell_{i,j}(\mathbf{W})\rVert_F\\
& \textstyle \leq \sqrt{K}\cdot \frac{1}{n}\sum_{i=1}^m \sum_{j=1}^{n_i} \ell_{i,j}(\mathbf{W})=\sqrt{K}\mathcal{L}(\mathbf{W}).\numberthis
\end{align*}

Next, we bound the Hessian. Observe that for the classes $k, l\in[K]$, we have 
\begin{equation}
    \nabla^2_{\mathbf{w}_k, \mathbf{w}_l}\ell_{i,j} (\mathbf{W})=
    \begin{cases}
        \widehat{y}_{i,j | {\bf W} }^{(k)} (1-\widehat{y}_{i,j | {\bf W} }^{(k)})\mathbf{x}_{i,j}\mathbf{x}_{i,j}^\top,& \text{if }k=l,\\
        -\widehat{y}_{i,j | {\bf W} }^{(k)}\widehat{y}_{i,j | {\bf W} }^{(l)}\mathbf{x}_{i,j}\mathbf{x}_{i,j}^\top,& \text{if }k\neq l.
    \end{cases}
\end{equation}
For the diagonal block of the Hessian, we have 
\begin{align*}
\|\nabla^2_{\mathbf{w}_k,\mathbf{w}_k}\ell_{i,j}(\mathbf{W})\|_{\text{op}}&=\widehat{y}_{i,j | {\bf W} }^{(k)}(1-\widehat{y}_{i,j | {\bf W} }^{(k)})\|\mathbf{x}_{i,j}\|^2\\
&\leq \widehat{y}_{i,j | {\bf W} }^{(k)}(1-\widehat{y}_{i,j | {\bf W} }^{(k)}).\numberthis\end{align*}
For the off-diagonal blocks on row $k$, we have 
\begin{align*}
    \sum_{l\neq k}\lVert \nabla^2_{\mathbf{w}_k,\mathbf{w}_l} \ell_{i,j}(\mathbf{W})\rVert_{\text{op}}&\leq \sum_{l\neq k}\widehat{y}_{i,j | {\bf W} }^{(k)}\widehat{y}_{i,j | {\bf W} }^{(l)}\|\mathbf{x}_{i,j}\|^2\\
    &=\widehat{y}_{i,j | {\bf W} }^{(k)}(1-\widehat{y}_{i,j | {\bf W} }^{(k)}).\numberthis
\end{align*}
Now, we use the Gershgorin Circle Theorem to get
\begin{align*}
& \lVert \nabla^2 \ell_{i,j} (\mathbf{W})\rVert_{\text{op}} \\& \leq \max_{k} \Big(\lVert \nabla^2_{\mathbf{w}_k,\mathbf{w}_k} \ell_{i,j}(\mathbf{W})\rVert_{\text{op}}+\sum_{l\neq k}\lVert \nabla^2_{\mathbf{w}_k,\mathbf{w}_l} \ell_{i,j}(\mathbf{W})\rVert_{\text{op}} \Big) \\
& \leq\max_k 2\, \widehat{y}_{i,j | {\bf W} }^{(k)}(1-\widehat{y}_{i,j | {\bf W} }^{(k)}).\numberthis
\end{align*} 
If $k=y_{i,j}$, then 
\begin{equation}
    \widehat{y}_{i,j | {\bf W} }^{(k)}(1-\widehat{y}_{i,j | {\bf W} }^{(k)})\leq 1-\widehat{y}_{i,j | {\bf W} } \leq -\log \widehat{y}_{i,j | {\bf W} } =\ell_{i,j}(\mathbf{W}).
\end{equation} 
Otherwise, we have 
\begin{align*}
    \widehat{y}_{i,j | {\bf W} }^{(k)} & \leq 1-\widehat{y}_{i,j | {\bf W} } \leq \ell_{i,j}(\mathbf{W}) ,\numberthis \\
    \widehat{y}_{i,j | {\bf W} }^{(k)}(1-\widehat{y}_{i,j | {\bf W} }^{(k)}) & \leq (1-\widehat{y}_{i,j | {\bf W} } )(1-\widehat{y}_{i,j | {\bf W} }^{(k)}) \\
    &\leq 1-\widehat{y}_{i,j | {\bf W} } \leq \ell_{i,j}(\mathbf{W}).\numberthis
\end{align*}
This leads to the conclusion: 
\begin{align*}
    & \textstyle \lVert \nabla^2\mathcal{L}(\mathbf{W})\rVert_{\text{op}}  \leq \frac{1}{n}\sum_{i=1}^m \sum_{j=1}^{n_i} \lVert \nabla^2\ell_{i,j}(\mathbf{W})\rVert_{\text{op}} \\
    & \textstyle \leq \frac{1}{n}\sum_{i=1}^m \sum_{j=1}^{n_i} 2\ell_{i,j}(\mathbf{W})=2\mathcal{L}(\mathbf{W}).\numberthis
\end{align*}

\subsection{Proof of Lemma \ref{lmm:gradpot}}
By Lemma \ref{lem:linsep}-(ii), we observe that
\begin{align*}
& \langle \nabla \mathcal{L}_i(\mathbf{W}), \mathbf{W}_\star\rangle \leq -\frac{\gamma}{n_i}\sum_{j=1}^{n_i}\left(1- \widehat{y}_{i,j | {\bf W}} \right) =-\gamma \mathcal{G}_i(\mathbf{W}).\numberthis
\end{align*}
Since $\mathcal{G}_i(\mathbf{W}) \geq 0$, we proceed by
\begin{align*}
& \langle \mathbf{W}^{(r+1)}, \mathbf{W}_\star\rangle \\
& = \langle \mathbf{W}^{(r)}, \mathbf{W}_\star\rangle -\frac{\bar{\eta}}{n}\sum_{i=1}^m \sum_{\tau=1}^{T_i} n_i \langle \nabla\mathcal{L}_i(\mathbf{D}^{(r)}_{i,\tau-1}), \mathbf{W}_\star\rangle \\
& \geq \langle \mathbf{W}^{(r)}, \mathbf{W}_\star\rangle + \frac{\bar{\eta} \gamma}{n} \sum_{i=1}^m \sum_{\tau=1}^{T_i} n_i \mathcal{G}_i(\mathbf{D}^{(r)}_{i,\tau-1}) \\
& \geq \langle \mathbf{W}^{(r)}, \mathbf{W}_\star\rangle +\frac{\bar{\eta} \gamma}{n} \mathcal{G}(\mathbf{W}^{(r)}).\numberthis
\end{align*} 
This yields
\begin{align*}
& \frac{1}{R}\sum_{r=0}^{R-1} \mathcal{G}(\mathbf{W}^{(r)})  \leq \frac{\langle \mathbf{W}^{(R)}, \mathbf{W}_\star\rangle_F-\langle \mathbf{W}^{(0)}, \mathbf{W}_\star\rangle_F}{\gamma\eta R} .\numberthis
\end{align*} 
Eq.~\eqref{eq:gradpot_bound} is then obtained using \eqref{eq:param_bound} and $\|\mathbf{W}_\star\|_F \le 1$.

By Lemma \ref{lem:linsep}-(ii), we also have
\begin{align*}
    &\langle \nabla \mathcal{L}(\mathbf{W}), \mathbf{W}_\star\rangle \leq
    -\gamma \mathcal{G}(\mathbf{W}).\numberthis
\end{align*}
Since ${\cal G}( {\bf W} ) \geq 0$, we have $|\langle \nabla \mathcal{L}(\mathbf{W}), \mathbf{W}_\star\rangle | \geq \gamma {\cal G} ( {\bf W} )$. 
By Cauchy-Schwarz inequality, this implies
\begin{align*}
    \lVert \nabla \mathcal{L}(\mathbf{W})\rVert_F & \geq \lVert \nabla \mathcal{L}(\mathbf{W})\rVert_F\cdot \lVert\mathbf{W}_\star\rVert_F\\
    & \geq |\langle \nabla \mathcal{L}(\mathbf{W}), \mathbf{W}_\star\rangle|\geq \gamma \mathcal{G}(\mathbf{W}).\label{eqn:nablaLnG}\numberthis
\end{align*}
Suppose that in addition, $\mathcal{G}(\mathbf{W})\leq 1/(2n)$, then since each term $(1-\widehat{y}_{i,j | {\bf W}})$ is non-negative, we must have 
\begin{equation}
    1-\widehat{y}_{i,j | {\bf W}}\leq n\cdot \mathcal{G}(\mathbf{W})\leq 1/2
\end{equation} 
for all $i, j$. Thus, $\widehat{y}_{i,j | {\bf W}} \geq 1/2$ for all $i,j$.
For any $z\in(0,1]$, we have the inequality
$-\log (z)\leq (1-z)/z$ and thus
\begin{align*}
\ell_{i_j}(\mathbf{W}) & \leq \frac{1-\widehat{y}_{i,j | {\bf W}}}{\widehat{y}_{i,j | {\bf W}}}\leq 2\left(1-\widehat{y}_{i,j | {\bf W}}\right). \numberthis
\end{align*}
Henceforth, we can derive that 
\begin{align*}
    \mathcal{L}(\mathbf{W}) & \leq \frac{1}{n}\sum\limits_{i=1}^m \sum\limits_{j=1}^{n_i} 2\left(1-\widehat{y}_{i,j | {\bf W}}\right)=2\mathcal{G}(\mathbf{W})\numberthis,
\end{align*} and when combined with (\ref{eqn:nablaLnG}), this implies that
$\|\nabla\mathcal{L}(\mathbf{W})\|_F\geq \gamma \mathcal{G}(\mathbf{W})\geq \frac{\gamma}{2}\mathcal{L}(\mathbf{W})$.

\section{Additional Experiments}

\subsection{Hard Sample Placement on Heterogeneous Devices}

The heterogeneity condition in Theorem~\ref{thm:stabled} is usually sufficient but conservative, as seen in Section~\ref{sec:exp}. In this experiment, we argue for the necessity of imposing such a condition by considering the interaction of the misaligned gradient updates among clients.

The basic setting of the empirical experiment follows the description as in Section \ref{sec:exp}, except that we generate $n_{\text{hard}} = 625$ hard samples with margin in $[\gamma, 2\gamma]$, where $\gamma = 0.001$, and $n_{\text{easy}} = 4375$ easy samples with margin at least $0.1$. 

All $m=8$ clients perform $T_i=57$ local SGD steps in the homogenerous setup, and clients $1$-$4$ perform $T_i=13$ and clients $5$-$8$ perform $T_i=101$ in the heterogeneous setup, yielding $T_\mathrm{avg}=57$ in both cases. 

The $625$ hard samples are assigned entirely to a single client. The $4375$ easy samples are partitioned equally among the remaining seven clients. Every client therefore holds exactly $625$ samples. We therefore consider three configurations (sharing the same set of data with different distributions) and the same average local work $T_{\mathrm{avg}} = 57$:
\begin{itemize}
    \item \texttt{hardLowT}: using the heterogeneous setup, we assign the hardest samples to an arbitrarily selected straggler with $T_i = 13$.
    \item \texttt{hardHighT}: using the heterogeneous setup, we assign the hardest samples to an arbitrarily selected fast client with $T_i = 101$.
    \item \texttt{homo}: using the homogeneous setup, we assign the hardest sampples to an arbitrarily selected client with $T_i=57$.
\end{itemize}

The \texttt{FedAvg} algorithm is then run for $R = 1,200,000$ communication rounds, with global stepsizes $\eta_g \in \{32, 64\}$ and local stepsize $\eta_l = 1.0$. 

\begin{figure}[!t]
    \centering
    \includegraphics[width=1.0\linewidth]{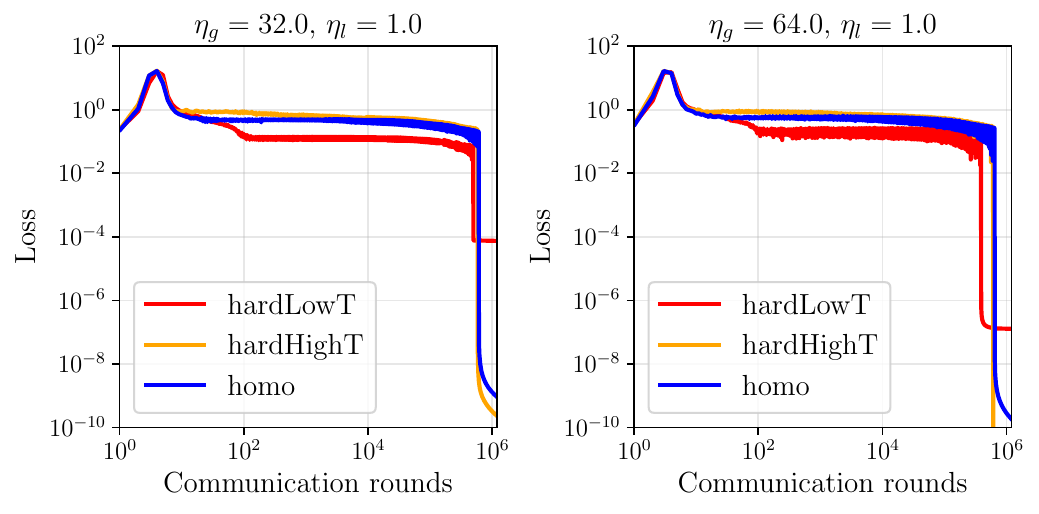}
    \caption{Effect of hard-sample placement under heterogeneous local step counts. All configurations share the 
    same set of data and the same average local work $T_{\mathrm{avg}} = 57$. The hardest samples are assigned to either straggler clients with $T_i = 13$ (\texttt{hardLowT}), fast clients with $T_i = 101$ (\texttt{hardHighT}) under the heterogeneous setup, or arbitrarily for the homogeneous setup where all $T_i = 57$ (\texttt{homo}). We use the global stepsizes $\eta_g \in \{32, 64\}$, local stepsize $\eta_l = 1$.}
    \label{fig:bulkhard}
\end{figure}

Fig.~\ref{fig:bulkhard} shows the loss progression for the three configurations, where the three configurations converge at different rates with \texttt{hardHighT} being the fastest, \texttt{homo57} being the intermediate, and \texttt{hardLowT} being significantly slower. While Corollary~\ref{thm:stabled} predicts a rate governed by $T_\mathrm{avg}$ alone, the observed ordering therefore reflects an effect the corollary does not capture.  In \texttt{hardHighT}, the hard client has $T_h=101$ and dominates the contribution to $\tilde{\mathbf{E}}$, which is largely aligned with $\nabla\mathcal{L}(\mathbf{E}^{(r)})$. This amplifies the effective descent direction and accelerates convergence relative to the homogeneous baseline. In \texttt{hardLowT} case, the hard client has $T_h=13$, where $\tilde{\mathbf{E}}$ is anti-aligned. The effective gradient is therefore reduced and the convergence rate decreases. In comparison, for the \texttt{homo57} case, there is no presence of client drift and the convergence rate is simply controlled by the same work conducted by each client. We hereby observe that the worst-case asymptotic rate is governed not by $T_\mathrm{avg}$ alone but by the alignment between step-count weights and gradient deviations. The rate can be either faster or slower than the homogeneous baseline, depending on the placement of the hard sample influencing the extent of the gradient alignment, justifying the necessity of the heterogeneity condition.

\subsection{Effect of Large Stepsize under Minibatch SGD Local Updates}

The concurrent work of Emmanouilidis et al. \cite{emmanouilidis2026sgdedgestabilitystochastic} shows that in the centralized setting, SGD on the multiclass cross-entropy loss self-stabilizes and converges in the best-iterate sense even with large stepsizes. A natural question is whether the stochastic stabilization mechanism is preserved in the federated setting, where the local updates are performed on heterogeneous data with heterogeneous number of local steps. To this end, we provide empirical evidence that the qualitative behavior predicted by our theory persists under local SGD.

The basic setting of the empirical experiment follows the description as in Section \ref{sec:exp}. We partition the synthetic dataset i.i.d. uniformly among $m=8$ clients, and we again consider the homogeneous configuration and heterogeneous configuration. All $8$ clients perform $T_i=57$ local SGD steps in the homogenerous configuration; in the heterogeneous configuration, clients $1$-$4$ perform $T_i=13$ and clients $5$-$8$ perform $T_i=101$, yielding $T_\mathrm{avg}=57$ in both cases.

We then consider batch sizes $b\in\{5,25,125\}$, stepsizes $\eta_g\in\{32,64\}$, $\eta_l=1$, and run \texttt{FedAvg} for $R=500,000$ rounds. At the beginning of each communication round, each client draws a fresh random permutation of its local dataset and partitions it into minibatches of size $b$. The client then performs $T_i$ local SGD updates, consuming one minibatch per step. Since $b$ divides the local dataset size exactly, every epoch boundary coincides with the end of a minibatch. When the permutation is exhausted, the client reshuffles and continues.

As shown in Fig.~\ref{fig:sgd}, for a sufficiently large $R$, the homogeneous and heterogeneous curves become asymptotically indistinguishable when $T_\mathrm{avg}$ is matched. Increasing the batch size shortens the transient for most cases, but it does not affect the asymptotic rate. Extending the centralized analysis of \cite{emmanouilidis2026sgdedgestabilitystochastic} to the federated setting is an interesting direction for future work.

\begin{figure}[!t]
    \centering
    \includegraphics[width=1.0\linewidth]{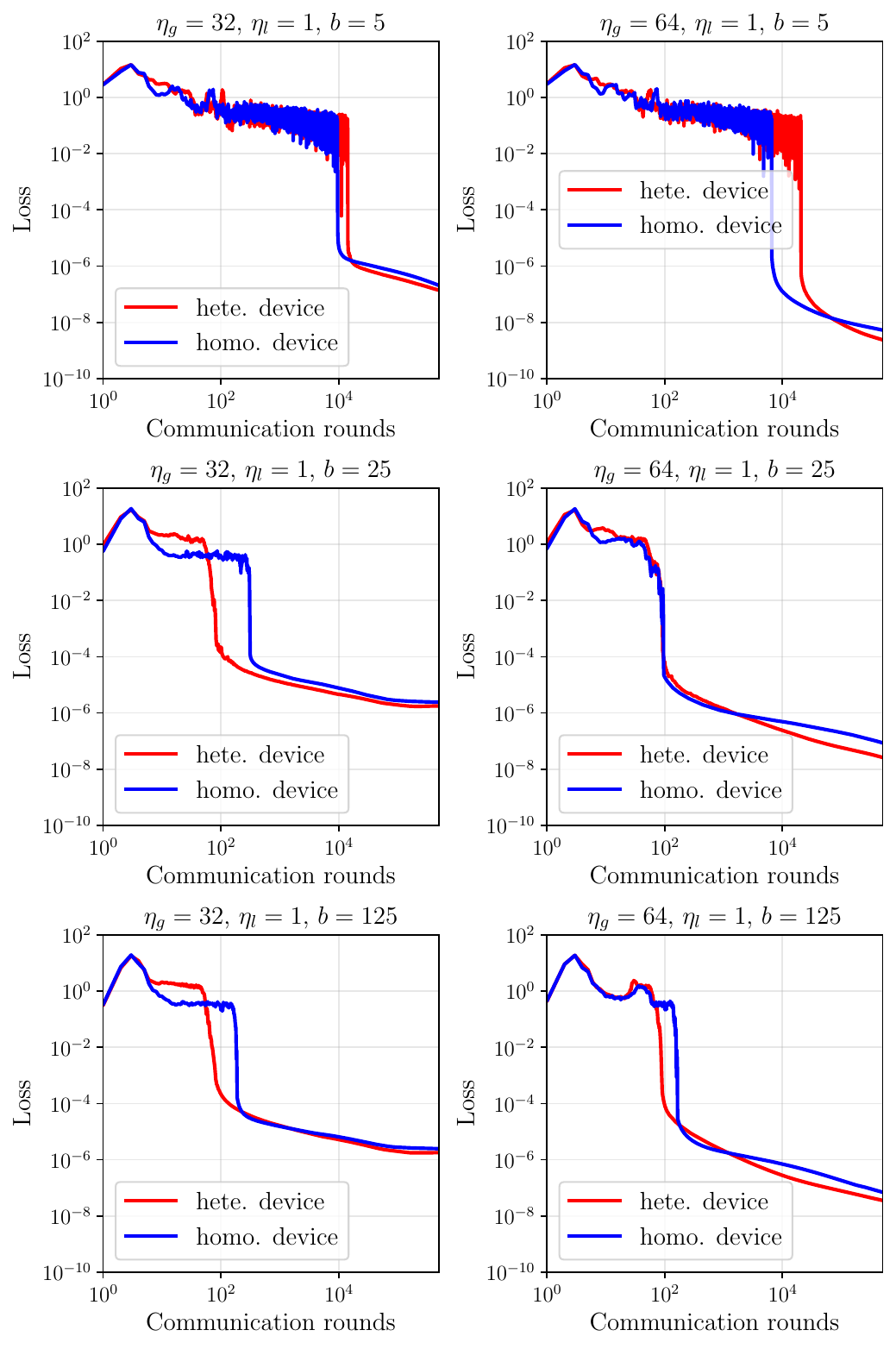}
    \caption{Effect of heterogeneous local update steps under minibatch SGD. All devices perform local updates using minibatches of size $b\in\{5, 25,125\}$, with global stepsizes $\eta_g \in \{32, 64\}$ and local stepsize $\eta_l = 1$. Two configurations are compared at a matched average $T_{\mathrm{avg}} = 57$: homogeneous ($T_i = 57$ for all devices) and heterogeneous ($T_i = 13$ for devices 1-4, $T_i = 101$ for devices 5-8).}
    \label{fig:sgd}
\end{figure}

Our result shows that \texttt{FedAvg} with local SGD exhibits the same qualitative behavior as local GD, where the iterates enter a stable phase while the loss decreases monotonically with an asymptotic rate governed by $T_\mathrm{avg}$.

\end{document}